\documentclass[10pt,journal,compsoc]{IEEEtran}
\ifCLASSOPTIONcompsoc
  \usepackage[nocompress]{cite}
\else
  \usepackage{cite}
\fi
\ifCLASSINFOpdf
\else
\fi
\usepackage{microtype}
\usepackage{graphicx}
\usepackage{subcaption}
\usepackage{booktabs} 

\usepackage{url}            
\usepackage{amsfonts}       
\usepackage{nicefrac}       
\usepackage{amsthm}
\usepackage{amsmath}
\usepackage{bbm}
\usepackage{breakcites}
\usepackage[overload]{empheq}
\usepackage[ruled,vlined]{algorithm2e}
\usepackage[table,xcdraw]{xcolor}
\usepackage{wrapfig}

\newtheorem{proposition}{Proposition}[]
\newtheorem{lemma}{Lemma}[]

\begin{document}
%
\title{Imitation Learning of Neural Spatio-Temporal Point Processes}
%
%
%
%

\author{Shixiang~Zhu,
        Shuang~Li, 
        Zhigang Peng, 
        Yao~Xie
\IEEEcompsocitemizethanks{\IEEEcompsocthanksitem Shixiang Zhu, Shuang Li, and Yao Xie were with H. Milton Stewart School of Industrial and Systems Engineering, Georgia Institute of Technology, Atlanta,
GA, 30332. Zhigang Peng was with School of Earth and Atmospheric Sciences, Georgia Institute of Technology, Atlanta,
GA, 30332\protect\\
}
}

\IEEEtitleabstractindextext{%
\begin{abstract}

We present a novel Neural Embedding Spatio-Temporal (NEST) point process model for spatio-temporal discrete event data and develop an efficient imitation learning (a type of reinforcement learning) based approach for model fitting. Despite the rapid development of one-dimensional temporal point processes for discrete event data, the study of spatial-temporal aspects of such data is relatively scarce. Our model captures complex spatio-temporal dependence between discrete events by carefully design a mixture of heterogeneous Gaussian diffusion kernels, whose parameters are parameterized by neural networks. This new kernel is the key that our model can capture intricate spatial dependence patterns and yet still lead to interpretable results as we examine maps of Gaussian diffusion kernel parameters. The imitation learning model fitting for the NEST is more robust than the maximum likelihood estimate. It directly measures the divergence between the empirical distributions between the training data and the model-generated data. Moreover, our imitation learning-based approach enjoys computational efficiency due to the explicit characterization of the reward function related to the likelihood function; furthermore, the likelihood function under our model enjoys tractable expression due to Gaussian kernel parameterization. Experiments based on real data show our method's good performance relative to the state-of-the-art and the good interpretability of NEST's result.

\end{abstract}

\begin{IEEEkeywords}
Spatio-temporal point processes, Generative model, Imitation learning
\end{IEEEkeywords}}

\maketitle

\IEEEdisplaynontitleabstractindextext

%
\IEEEpeerreviewmaketitle


\section{Introduction}
\label{sec:introduction}

Spatio-temporal event data has become ubiquitous, emerging from various applications such as social media data, crime events data, social mobility data, and electronic health records. Such data consist of a sequence of times and locations that indicate when and where the events occurred. Studying generative models for discrete events data has become a hot area in machine learning and statistics: it reveals the pattern of the data, helps us to understand the data dynamic and information diffusion, as well as serves as an essential step to enable subsequent machine learning tasks. 

Point process models (see \cite{Reinhart2018} for an overview) have become a common choice for generative models of discrete event data. In particular, the self and mutual exciting processes, also known as the Hawkes processes, are popular since they can capture past events' influence on future events over time, space, and networks. 

Despite the rapid development of one-dimensional temporal point processes models for discrete event data, studies focusing on {\it spatial-temporal} aspects of such data are relatively scarce. The original works of \cite{Ogata1988, Ogata1998} develop the so-called epidemic-type aftershock sequence (ETAS) model, which is still widely used, suggesting an exponential decaying diffusion kernel function. 
This model captures the seismic activities' mechanism and is convenient to fit, as the kernel function is homogeneous at all locations with the same oval shape (Fig.~\ref{fig:diffusion-exp}). However, these classical models for spatio-temporal event data (usually statistical models in nature) tend to make strong parametric assumptions on the conditional intensity.

\begin{figure}[t!]
\centering
\includegraphics[width=.7\linewidth]{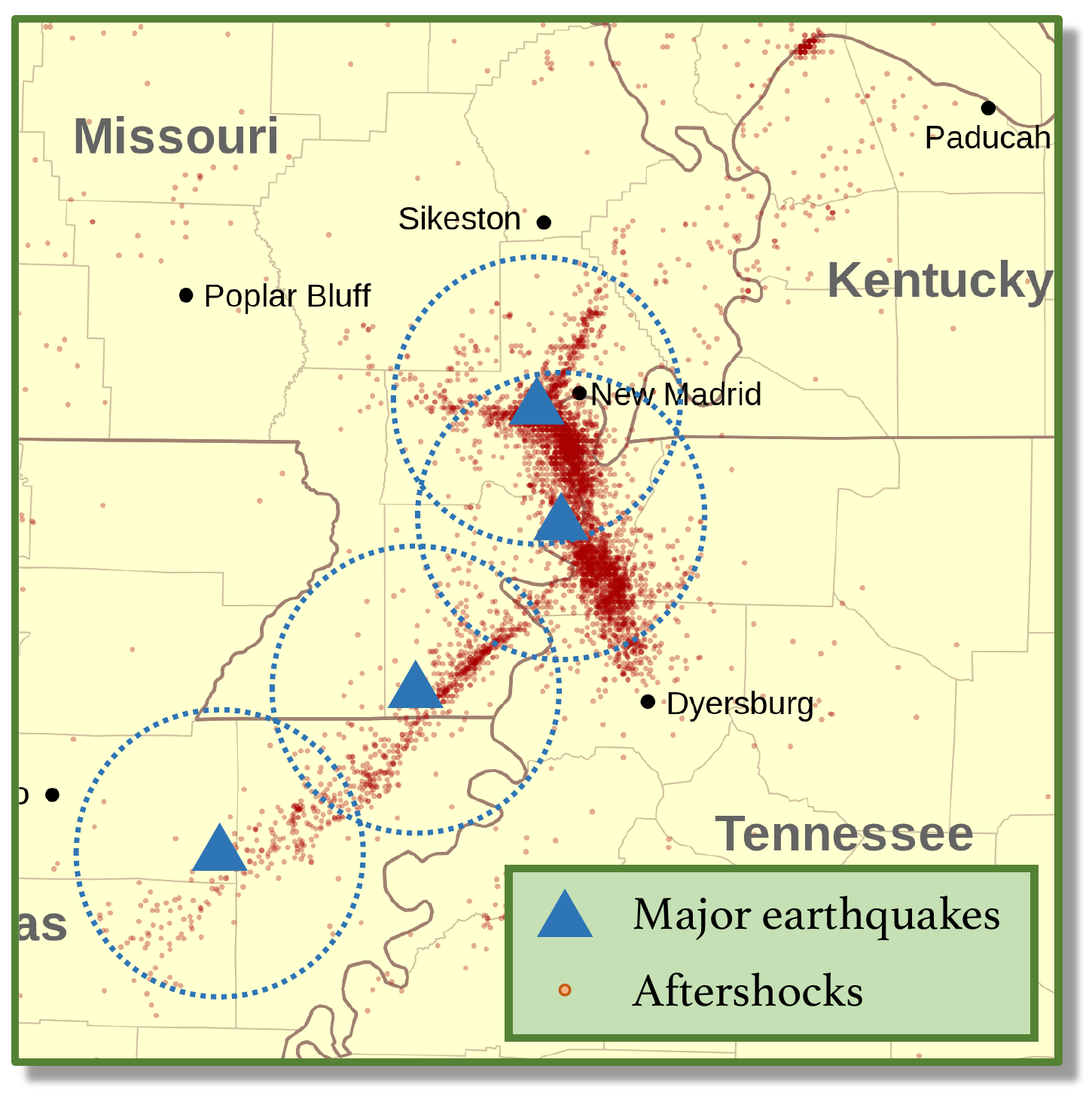}
\caption{
A motivating example of seismic activities: four earthquakes and their aftershocks occurred in New Madrid, MO., in the United States since 1811. The blue triangles represent the major earthquakes, and the dotted circles represent the estimated aftershock regions suggested by the ETAS. The small red dots represent the actual aftershocks caused by the major earthquakes. We can observe that the locations of actual aftershocks are related to the geologic structure of faults, and the vanilla ETAS model may not sufficiently capture such complex spatial dependence.
}
\vspace{-0.1in}
\label{fig:earthquake}
\end{figure}

\begin{figure*}[t!]
\centering
\includegraphics[width=.8\linewidth]{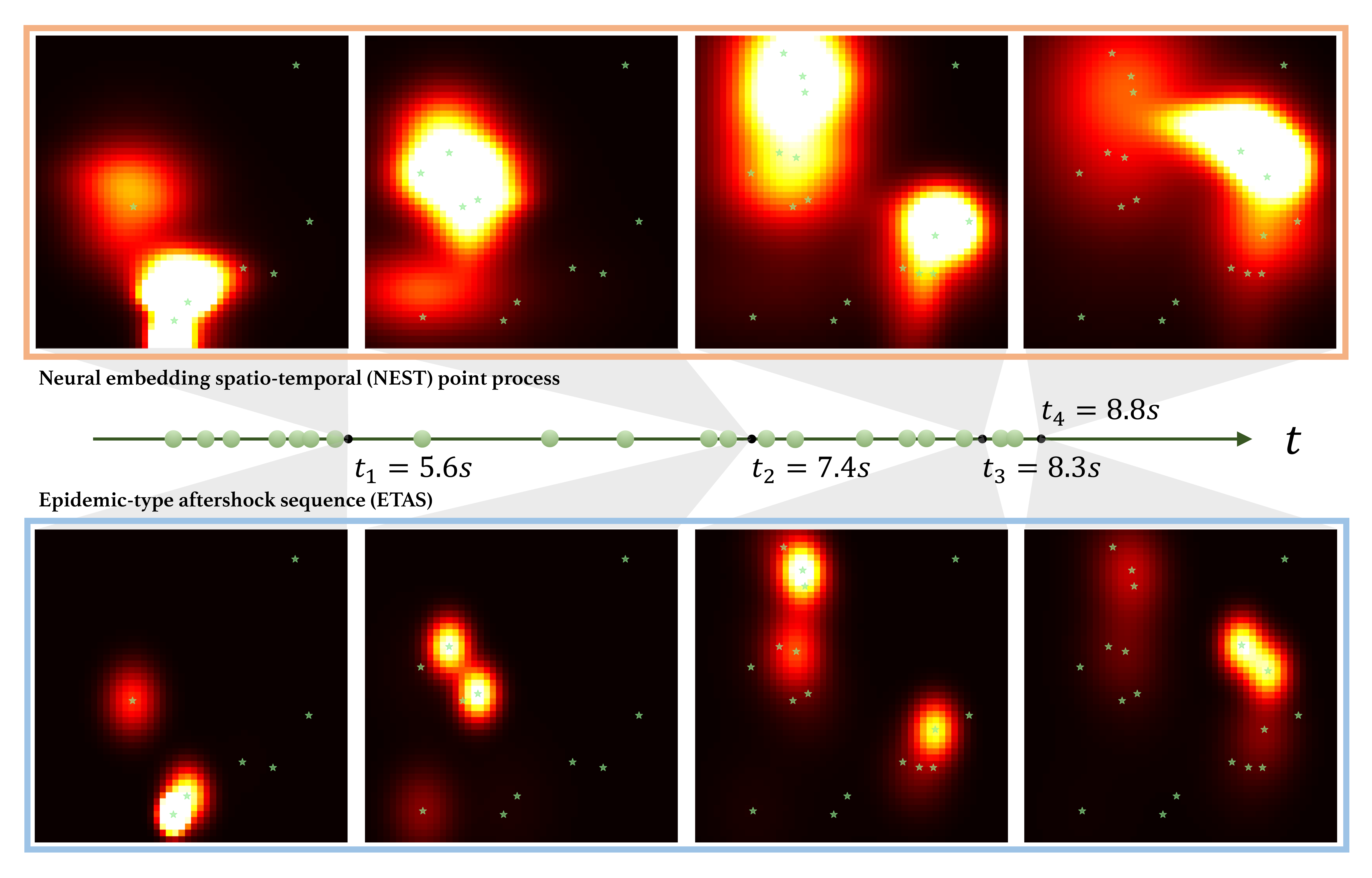}
\caption{Comparison between the standard epidemic-type aftershock sequence (ETAS) and the proposed neural embedding spatio-temporal (NEST) point process models, fitted to the same set of police 911 calls-for-service data. We visualize snapshots of the conditional intensities $\lambda^*(t, s), s \in \mathcal{S}$ at four times for the same sequence of events that occurred in a square region $\mathcal{S} = [-1, +1] \times [-1, +1]$; brighter areas are higher intensities where the next event is more likely to occur.  The green dots on the timeline represent when the events occur, and the green stars in the square region represent the events' locations. Snapshots are taken at $t_1$, $t_2$, $t_3$, and $t_4$, which are indicated by the black dots on the timeline. As self-exciting point processes, the occurrence of a new event will abruptly increase the local intensity, and the influence will decay and diffuse to the surrounding regions over time. 
As we can see from the comparison, the NEST can capture more intricate spatial dynamics than the ETAS.}
\label{fig:diffusion-exp}
\end{figure*}

However, in specific scenarios, the simplified spatio-temporal models based on the ETAS may lack flexibility. It does not capture the anisotropic spatial influence or other intricate spatial dependence structure. Take earthquake event data as an example, consisting of a sequence of records of seismic activities: their times and locations. The aftershocks are minor seismic activities that are trigger by the major earthquakes. According to the study \cite{Stein2009},
it has been shown that most of the recent seismic events that occurred in New Madrid, MO, are aftershocks of four earthquakes of magnitude 7.5 in 1811. As shown in Fig. \ref{fig:earthquake}, the distribution of the minor seismic activities is in a complex shape (clearly not ``circles" or isotropic), and the spatial correlation between seismic activities is related to the geologic structure of faults through the complex physical mechanism and usually exhibits a heterogeneous conditional intensity. For instance, most aftershocks occur along the fault plane or other faults within the volume affected by the mainshock's strain. This creates a spatial profile of the intensity function that we would like to capture through the model, such as the direction and shape of the intensity function at different locations, to provide useful information to geophysicists' scientific study. 

On the other hand, when developing spatio-temporal models, we typically want to generate some statistical interpretations (e.g., temporal correlation, spatial correlation), which may not be easily derived from a complete neural network model. Thus, a generative model based on specifying the conditional intensity of point process models is a popular approach. For example, recent works \cite{Du2016,Mei2017, Li2018, Upadhyay2018, Xiao2017A, Xiao2017B, Zhu2019B} has achieved many successes in modeling temporal event data  (some with marks) which are correlated in time. 
It remains an open question on extending this type of approach to include the spatio-temporal point processes. One challenge is how to address the computational challenge associated with evaluating the log-likelihood function (see expression \eqref{eq:mle}). This can be intractable for the general models without a carefully crafted structure since it requires integrating the conditional intensity function in a continuous spatial and temporal space. Another challenge is regarding how to develop robust model-fitting methods that do not rely too much on the modeling assumption. 


This paper first presents a novel neural embedding spatio-temporal (NEST) point process for spatio-temporal discrete event data. Our proposed NEST model tackles flexible representation for complex spatial dependence, interpretability, and computational efficiency, through meticulously designed neural networks capturing spatial information. We generalize the idea of using a Gaussian diffusion kernel to model spatial correlation by introducing the more flexible heterogeneous mixture of Gaussian diffusion kernels with shifts, rotations, and anisotropic shapes. Such a model can still be efficiently represented using a handful of parameters (compared with a full neural network model such as convolutional neural networks (CNN) over space). The Gaussian diffusion kernels are parameterized by neural networks, which allows the kernels to vary continuously over locations. This is the key that our model can capture intricate spatial dependence patterns and yet still lead to interpretable results as we examine maps of Gaussian diffusion kernel parameters. As shown in Fig.~\ref{fig:diffusion-exp}, our model is able to represent arbitrary diffusion shape at different locations in contrast to the ETAS developed by \cite{Ogata1981, Ogata1988,Ogata1998}. 

Second, we develop computationally efficient approaches for fitting the NEST model based on imitation learning (IL) \cite{Gretton2007, Li2018}, and compare it with the maximum likelihood estimation (MLE) method. The imitation learning model fitting for the NEST is more flexible and robust. It directly measures the divergence between the empirical distributions between the training data and the model-generated data. Moreover, our imitation learning-based approach enjoys computational efficiency due to the explicit characterization of the reward function related to the likelihood function; furthermore, the likelihood function under our model enjoys tractable expression due to Gaussian kernel parameterization. Experiments based on synthetic and real data show our method's superior performance in event prediction and model interpretability compared to the state-of-the-art.


The rest of the paper is organized as follows. We start by reviewing related literature. Then Section~\ref{sec:background} introduces the background of spatio-temporal point processes and related classic models. We present our proposed NEST model in Section~\ref{sec:model} and describe our imitation learning framework for model fitting in Section~\ref{sec:learning}. Section~\ref{sec:experiments} contains experimental results based on synthetic and real data. Finally, Section \ref{sec:conclusion} concludes the paper with discussions.

\subsection{Related work}

Existing literature on spatial-temporal point process modeling usually makes simplified parametric assumptions on the conditional intensity based on the prior knowledge of the processes. Landmark works \cite{Ogata1998, Ogata1988} suggest an exponential decaying kernel function. This model captures seismic activities' mechanism to a certain degree and is easy to learn, as the kernel function is homogeneous at all locations.
However, in applications to other scenarios, such a parametric model may lack flexibility.

Another approach to obtaining generative models for temporal point processes is based on the idea of imitation learning and reinforcement learning. Good performance has been achieved for modeling temporal point processes \cite{Li2018} and marked temporal point processes \cite{Upadhyay2018}. 
The premise is to formulate the generative model as a policy for reinforcement learning and extract policy from sequential data as if it were obtained by reinforcement learning \cite{Richard1998} followed by inverse reinforcement learning \cite{Andrew2000, Ho2016}. In this way, the generative model is parameterized by neural networks \cite{Du2016, Duan2016, Mei2017, Volodymyr2016}. However, representing the conditional intensity entirely using neural networks may lack certain interpretability. Compressing all the information by neural networks may also miss the opportunity to consider prior knowledge about the point processes. Also, it remains an open question on how to extend this approach to include the spatial component. The spatial-temporal point processes are significantly more challenging to model than the temporal point processes since the spatio-temporal dependence is much more intricate than the one-dimensional temporal dependence. 

Other works such as \cite{Short2010-1, Short2010-2} have achieved some success in modeling complicated spatial patterns of crime by considering the spatial influence as \textit{hotspots} with full parametric models. Some works \cite{Fox2016, Zipkin2016, Lewis2012, Xiao2017B} consider the event sequences as temporal point processes without incorporating spatial information leverages non-parametric approaches. As a compromise, \cite{Zhu2019A, Mohler2011, Mei2017} seek to learn the spatial and temporal pattern jointly by multivariate point processes with discretized the spatial structure. 

\section{Background}
\label{sec:background}

This section revisits the definitions of the spatio-temporal point processes (STPP) and one of the most commonly used STPP model -- epidemic-type aftershock sequence (ETAS).

\subsection{Spatio-temporal point processes (STPP)}

STPP consists of an ordered sequence of events in time and location space. Assume the time horizon is $[0, T]$ and the data is given by $\{a_1, a_2, \dots, a_{N(T)}\}$, which are a set of sequences of events ordered in time. Each $a_i$ is a spatio-temporal tuple $a_i = (t_i, s_i)$, where $t_i \in [0, T)$ is the time of the event and $s_i \in \mathcal{S} \subseteq \mathbb{R}^2$ is the associated location of the $i$-th event. We denote by $N(T)$ the number of the events in the sequence between time $[0, T)$ and in the region $\mathcal{S}$. 

The joint distribution of a STPP is completely characterized by its conditional intensity function $\lambda(t, s|\mathcal{H}_t)$. Given the event history $\mathcal{H}_t = \{ (t_i, s_i)|t_i < t \}$, the intensity corresponds to the probability of observing an event in an infinitesimal around $(t, s)$:
\begin{equation*}
    \lambda(t, s | \mathcal{H}_t)
    = \lim_{\Delta t, \Delta s \rightarrow 0} \frac{\mathbb{E}\left[ N([t, t+\Delta t) \times B(s, \Delta s)) | \mathcal{H}_t \right]}{\Delta t \times |B(s, \Delta s)|},
\end{equation*}
where $N(A)$ is the counting measure of events over the set $A \subseteq [0, T) \times \mathcal{S}$, $B(s, \Delta s)$ denotes a Euclidean ball centered at $s$ with radius $\Delta s$, and $|\cdot|$ is the Lebesgue measure. Below, for the notational simplicity, we denote the conditional intensity function $\lambda(t, s | \mathcal{H}_t)$ as $\lambda^*(t, s)$. 

For instance, a type of self-exciting point processes, Hawkes processes \cite{Hawkes1971} has been widely used to capture the mutual excitation among temporal events. 
Assuming that influence from past events are linearly additive for the current event, the conditional intensity function of a Hawkes process is defined as 
\[
    \lambda(t | \mathcal{H}_{t}) = \lambda_0 + \sum_{t_i < t} \nu(t - t_i),
\]
where $\lambda_0 \ge 0$ is the background intensity of events, $\nu(\cdot) \ge 0$ is the \emph{triggering function} that captures temporal dependencies of the past events. The triggering function can be chosen in advance, for instance, in the one-dimensional case $\nu(t - t_i) = \alpha \exp \{ - \beta (t - t_i) \}$.

\subsection{ETAS model} 

The most commonly used kernel function for spatio-temporal point processes is the standard {\it diffusion kernel} function proposed by epidemic type aftershock-sequences (ETAS) \cite{Musmeci1992}, which was initially introduced to model the earthquake events, but now widely used in many other applications \cite{Ogata1988, Ogata1998, Zhu2019A, Fox2016, Lewis2012, Zipkin2016}. 
ETAS model assumes that the influence over time and space decouples, and the influence decays exponentially over time. Over space decay only depends on distance (thus, it is a spherical model). Therefore, {\it The ETAS model does not capture the anisotropic shape of the kernel.} This is a simplification and may not capture complex spatial dependence. The ETAS model can also deal with scalar-valued marks (e.g., the magnitude of earthquakes), which we will not discuss here while only focusing on spatio-temporal interactions between events. We also note that one of the reasons that the ETAS is a popular model is due to its interpretability.

\section{Proposed model}
\label{sec:model}

\begin{figure*}[!t]
\centering
\includegraphics[width=1.\linewidth]{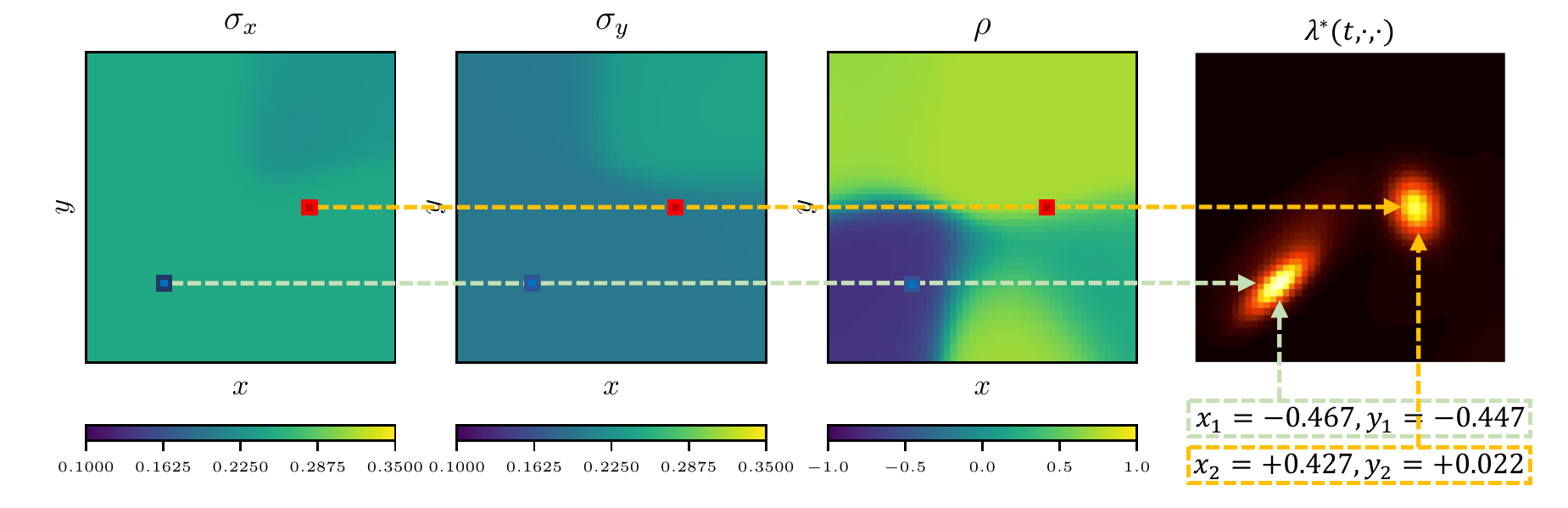}
\caption{An example of kernel used in the NEST model: $\sigma_x$, $\sigma_y$, $\rho$ defines a Gaussian component in the heterogeneous Gaussian diffusion kernel. The rightmost subfigure shows the conditional intensity at time $t$, where two points occurred at location $(x_1, y_1)$ and $(x_2, y_2)$ have triggered the two diffusions (the bright spots) with different shapes. This Gaussian component is specified by parameters (mean, covariance) that vary continuously over space, which are themselves represented by neural networks.}
\label{fig:exp-mappings}
\end{figure*}

To capture the complex and heterogenous spatial dependence in discrete events, we present a novel {\it continuous-time and continuous-space} point process model, called neural embedding spatio-temporal (NEST) model. The NEST uses flexible neural networks to represent the conditional intensity's spatial heterogeneity while retaining interpretability as a semi-parametric statistical model.

\subsection{Spatially heterogeneous Gaussian diffusion kernel} 

We start by specifying the conditional probability of the point process model, as it will uniquely specify the joint distribution of a sequence of events. First, to obtain a similar interpretation as the ETAS model \cite{Ogata1988}, we define a similar parametric form for the conditional intensity function
\begin{equation}
\lambda^*(t, s) = \lambda_0 + \sum_{j:t_j<t} \nu(t, t_j, s, s_j),
\label{eq:con_intensity}
\end{equation}
where $\lambda_0 > 0$ is a constant background rate, $\nu$ is the kernel function that captures the influence of the past events $\mathcal H_t$. The form of the kernel function $\nu$ determines the profile of the spatio-temporal dependence of events.

We assume the kernel function takes a standard Gaussian diffusion kernel over space and decays exponentially over time. We adopt a mixture of generalized Gaussian diffusion kernels to enhance the spatial expressiveness, which is location dependent. Thus, it can capture a more complicated spatial-nonhomogeneous structure. Given all past events $\mathcal H_t$, we define
\begin{equation}
\begin{aligned}
&\nu(t, t', s, s') = \sum_{k=1}^K \phi_{s'}^{(k)} \cdot g(t, t', s, s' | \Sigma_{s'}^{(k)}, \mu_{s'}^{(k)}), \\
&\quad\quad \forall t' <t, s \in \mathcal S, 
\label{eq:gaussian-mixture-kernel}
\end{aligned}
\end{equation}
where $K$ is a hyper-parameter that defines the number of components of the Gaussian mixture; $\mu_{s'}^{(k)} \in \mathbb{R}^2$ and $\Sigma_{s'}^{(k)} \in \mathbb{R}^{2 \times 2}$ are the mean and covariance matrix parameters for the $k$th Gaussian component at location $s'$; $\phi_{s'}^{(k)}: \mathcal S \rightarrow \mathbb R$ is the corresponding weight of the component that satisfies $\sum_{k=1}^K \phi_{s'}^{(k)} = 1$, $\forall s' \in \mathcal S$. The exact forms of $\mu_{s'}^{(k)}$, $\Sigma_{s'}^{(k)}$, and $\phi_{s'}^{(k)}$ will be specified later. In the following discussions, we focus on describing a single Gaussian component and omit the superscript $k$ for the notational simplicity. 
 
Now each Gaussian diffusion kernel is defined as
\begin{align*}
& g(t, t', s, s' | \Sigma_{s'}, \mu_{s'}) = \frac{C e^{-\beta (t - t')}}{2 \pi \sqrt{|\Sigma_{s'}|} (t - t')} \cdot \\
& \exp \Bigg \{- \frac{(s - s' - \mu_{s'})^T \Sigma_{s'}^{-1} (s - s' - \mu_{s'})}{2(t - t')} \Bigg\},
\end{align*}
where $\beta >0$ controls the temporal decay rate; $C>0$ is a constant that decides the magnitude; $\mu_{s} = [\mu_x(s), \mu_y(s)]^T$ and $\Sigma_s$ denote the mean and covariance parameters of the diffusion kernel; $|\cdot|$ denotes the determinant of a covariance matrix. 
Note that the structure of the kernel function $g(\cdot|\Sigma_s, \mu_s)$ may vary over ``source'' locations $s \in \mathcal S$. 
To be specific, $\Sigma_s$ is defined as a positive semi-definite matrix
\[
\Sigma_s = 
\begin{pmatrix}
    \sigma^2_x(s) &  \rho(s) \sigma_x(s) \sigma_y(s) \\
    \rho(s) \sigma_x(s) \sigma_y(s) & \sigma^2_y(s)
\end{pmatrix}.
\]

The parameters $\mu_s$ and $\Sigma_s$ control the shape (shift, rotation, et cetera) of each Gaussian component. As shown in Fig.~\ref{fig:exp-mappings}, parameters $\sigma_x(s), \sigma_y(s), \rho(s)$ may be different at different location $s$ and jointly control the spatial structure of the diffusion. Parameters $\mu_x(s), \mu_y(s)$ define the offset of the center of the diffusion from the location $s$. Let $\mu_x: \mathcal{S} \rightarrow \mathbb{R}$, $\mu_y: \mathcal{S} \rightarrow \mathbb{R}$, $\sigma_x: \mathcal{S} \rightarrow \mathbb{R}^+$, $\sigma_y: \mathcal{S} \rightarrow \mathbb{R}^+$, and $\rho: \mathcal{S} \rightarrow (-1, 1)$ be non-linear mappings from location space $\mathcal{S}$ to the corresponding parameter space. To capture intricate spatial dependence, we represent such non-linear mappings using neural networks. 



\subsection{Comparison with ETAS model} 

In the standard ETAS, the kernel function can be thought of as a special instance of the proposed heterogeneous Gaussian diffusion kernel in the NEST with {\it a single component whose parameters do not vary over space and time}: the kernel function defined in \eqref{eq:gaussian-mixture-kernel} can be simplified to
$\nu(t, t', s, s')  = g(t, t', s, s' |\Sigma, \mu)$, 
where the spatial and temporal parameters are location invariant, i.e., $\Sigma \equiv \mbox{diag}\{\sigma_x^2, \sigma_y^2\}$ and $\mu \equiv 0$. Compared with the standard Gaussian diffusion kernel used in ETAS, we introduce additional parameters $\rho, \mu_x, \mu_y$ that allows the diffusion to shift, rotate, or stretch in the space. A comparison between the spatio-temporal kernels used in the ETAS and the NEST is presented in Fig.~\ref{fig:diffusion-exp}. 

\subsection{Deep neural network representation} 

Recall that parameters in each Gaussian component are determined by a set of non-linear mappings $\{\rho(s), \sigma_x(s), \sigma_y(s), \mu_x(s), \mu_y(s)\}$. 
We capture these non-linear spatial dependencies using a deep neural network through a latent embedding. 

Assume the spatial structure at location $s$ can be summarized by a latent embedding vector $\boldsymbol{h}(s) \in \mathbb{R}^d$, where $d$ is the dimension of the embedding. 
The parameters of a Gaussian component at location $s$ can be represented by an output layer of the neural network. The input of the layer is the latent embedding $\boldsymbol h(s)$.
This layer is specified as follows. The mean parameters are specified by
\begin{align*}
    \mu_x(s) & = C_x \cdot \left( \text{sigm}(\boldsymbol{h}(s)^T W_{\mu_x} + b_{\mu_x}) - 1/2 \right), \\
    \mu_y(s) & = C_y \cdot \left( \text{sigm}(\boldsymbol{h}(s)^T W_{\mu_y} + b_{\mu_y}) - 1/2 \right), 
\end{align*}
where $C_x, C_y$ are preset constants that control the shift of the center of Gaussian components from location $s$, $\text{sigm}(x) = 1 / (1+e^{-x})$ is the sigmoid function which gives an output in the range $[0, 1]$. The variance and the correlation parameters are specified by 
\begin{align*}
  \sigma_x(s) & = \text{softplus}(\boldsymbol{h}(s)^T W_{\sigma_x} + b_{\sigma_x}), \\
    \sigma_y(s) & = \text{softplus}(\boldsymbol{h}(s)^T W_{\sigma_y} + b_{\sigma_y}), \\
        \rho(s) & = 2 \cdot \text{sigm}(\boldsymbol{h}(s)^T W_{\rho} + b_{\rho}) - 1,
\end{align*}
where $\text{softplus} = \log (1 + e^x)$ is a smooth approximation of the ReLU function. 
The parameters in the network $\theta_w = \{ W_{\sigma_x}, W_{\sigma_y}, W_{\mu_x}, W_{\mu_y}, W_{\rho}\}$ and $\theta_b = \{ b_{\sigma_x}, b_{\sigma_y}, b_{\mu_x}, b_{\mu_y}, b_{\rho} \}$ are weight-vectors and biases in the output layer of the Gaussian component.
Note that we omitted the superscript $k$ of each parameter in the discussion above for notational simplicity. However, it should be understood that each Gaussian component will have its own set of parameters.
Finally, the weight of each component is given by $\phi_{s}^{(k)}$, which is defined through the soft-max function
\[
\phi_s^{(k)} = e^{\boldsymbol{h}(s)^T W_\phi^{(k)}} \Big/ \sum_{\kappa=1}^K e^{\boldsymbol{h}(s)^T W_\phi^{(\kappa)}}.
\]
where $W_\phi^{(k)} \in \mathbb{R}^{d}$ is a weight vector to be learned.
The latent embedding $\boldsymbol h(s)$ is characterized by another neural network defined as
$
\boldsymbol{h}(s) = \psi(s|\theta_h)
$,
where $\psi(\cdot): \mathbb R^2\rightarrow \mathbb{R}^d$ is a fully-connected multi-layer neural network function taking spatial location $s$ as input;  $\theta_h$ contains the parameters in this neural network. In our experiments later, we typically use three-layer neural networks where each layer's width is 64.  

In summary, the NEST with heterogeneous Gaussian mixture diffusion kernel is jointly parameterized by $\theta = \{\beta, \theta_h, \{W_\phi^{(k)}, \theta_w^{(k)}, \theta_b^{(k)}\}_{k=1,\dots,K}\}$. The architecture is summarized in Fig.~\ref{fig:gaussian-mix-illustration}. In the following, we denote the conditional intensity as $\lambda_\theta^*(s, t)$ defined in (\ref{eq:con_intensity}), to make the dependence on the parameters more explicit. We emphasize that the Gaussian diffusion kernels' parameters vary {\it continuously} over location, and are represented by flexible neural networks; this is the key that our model can capture intricate spatial dependence in practice. 


\begin{figure*}
\centering
\includegraphics[width=1\linewidth]{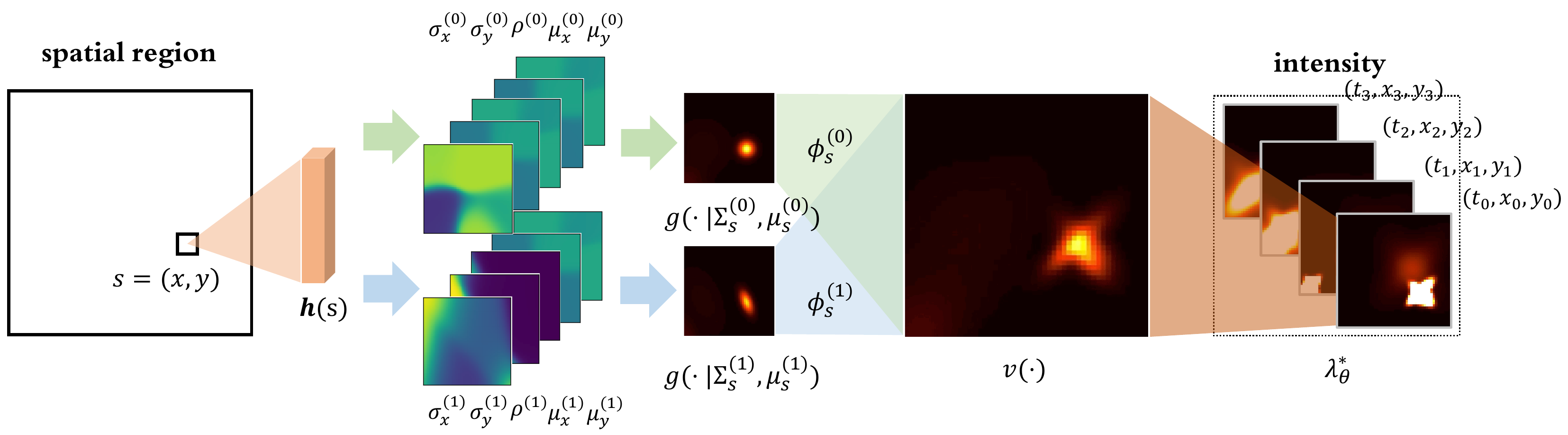}
\caption{An illustration for NEST's neural network architecture based on a mixture of heterogeneous Gaussian diffusion kernel. Note that each Gaussian kernel is specified by neural networks, which summarizes the latent embedding information from data.}
\label{fig:gaussian-mix-illustration}
\end{figure*}

\section{Computationally efficient learning}
\label{sec:learning}

In this section, we define two approaches to learn parameters for the NEST model: (1) the maximum likelihood-based approach, and (2) the imitation learning-based approach, using a policy parameterized by the conditional intensity and a non-parametric reward function based on the maximum mean discrepancy (MMD) metric \cite{Gretton2007}.

\subsection{Maximum likelihood approach}

The model parameters can be estimated via maximum likelihood estimate (MLE) since we have the explicit form of the conditional intensity function. Given a sequence of events $\boldsymbol{a} = \{a_0, a_1, \dots, a_{n}\}$ occurred on $(0, T] \times \mathcal{S}$ with length $n$, where $a_i = (t_i, s_i)$, the log-likelihood is given by 
\begin{equation}
\ell(\theta) = 
 \left( \sum_{i=1}^{n} \log \lambda^*_{\theta}(t_i, s_i) \right)  
- \int_{0}^{T} \int_{\mathcal{S}} \lambda^*_{\theta}(\tau, r) dr d\tau.
\label{eq:mle}
\end{equation}

A crucial step to tackle the computational challenge is to evaluate the integral in (\ref{eq:mle}).
Here, we can obtain a closed-form expression for the likelihood function using the following proposition. This can reduce the integral to an analytical form, which can be evaluated directly without numerical integration (see the proof in the appendix).
\begin{proposition}[Integral of conditional intensity function]\label{prop1}
Given ordered event times $0 = t_0 < t_1 < \cdots < t_n < t_{n+1} = T$, for $i = 0, \ldots, n$,
 \begin{equation*}
\begin{split} 
 & \int_{t_{i}}^{t_{i+1}} \int_{\mathcal{S}} \lambda^*_{\theta}(\tau, r) dr d\tau  = 
  \lambda_0 (t_{i+1} - t_i) |\mathcal{S}| \\
  & + (1-\epsilon) \frac{C}{\beta} \sum_{j:t_j < t_i} C_j \left(e^{-\beta(t_{i} - t_{j})} - e^{-\beta(t_{i+1}-t_{j})} \right),  
 \end{split}
 \end{equation*}
where 
\[
C_j = \sum_{k=1}^K \phi_{s_j}^{(k)} \frac{\sigma_{x}^{(k)}(s_j)\sigma_{y}^{(k)}(s_j)}{|\Sigma^{(k)}_{s_j}|^{1/2}},
\]
and the constant 
 \[\epsilon =\max_{j: t_j < t_{i+1}} \frac{\int_{t_i}^{t_{i+1}} \int_{\mathcal S} g(\tau, t_j, r, s_j) dr d \tau}{\int_{t_i}^{t_{i+1}}\int_{\mathbb R^2} g(\tau, t_j, r, s_j) dr d\tau}.\]
\end{proposition}
Since spatially, the kernel $g$ is a Gaussian concentrated around $s$, and most events $s_i$ locates in the relatively interior of $\mathcal S$ when $\mathcal S$ is chosen sufficiently large, we can ignore the marginal effect, and $\epsilon$ can become a number much smaller than 1. Due to the decreased activity in the region's edges, the boundary effect is usually negligible \cite{Ogata1998}.
Define $t_0 = 0$ and $t_{n+1} = T$. Since
\[
\int_{0}^{T} \int_{\mathcal{S}} \lambda^*_{\theta}(\tau, r) dr d\tau = \sum_{i=0}^{n+1} \int_{t_i}^{t_{i+1}} \int_{\mathcal{S}} \lambda^*_{\theta}(\tau, r) dr d\tau,
\] 
using Proposition \ref{prop1}, we can write down the integral in the log-likelihood function in a closed-form expression. 

Finally, the optimal parameters can be thus obtained by $\hat{\theta} = \text{argmax}_{\theta} \log \ell(\theta)$. Due to the non-convex nature of this problem, we solve the problem by stochastic gradient descent.

\subsection{Imitation learning approach}
\label{sec:imitation}

We now present a more flexible, imitation learning framework for model fitting. This approach's primary benefit is that it does not rely on the pre-defined likelihood function; the reward function can be learned in a data-driven manner to optimally ``imitate'' the empirical distribution of the training data \cite{langford2003reducing}. Hence, it is more robust to model mismatch. Moreover, we want to emphasize that, in our setting, the learned reward function can be represented in a closed-form expression and conveniently estimated using samples. Thus, we can avoid the expensive inverse reinforce learning part for general imitation learning problems.

%
%

\begin{figure}[!b]
\centering
\begin{subfigure}[h]{1\linewidth}
\includegraphics[width=\linewidth]{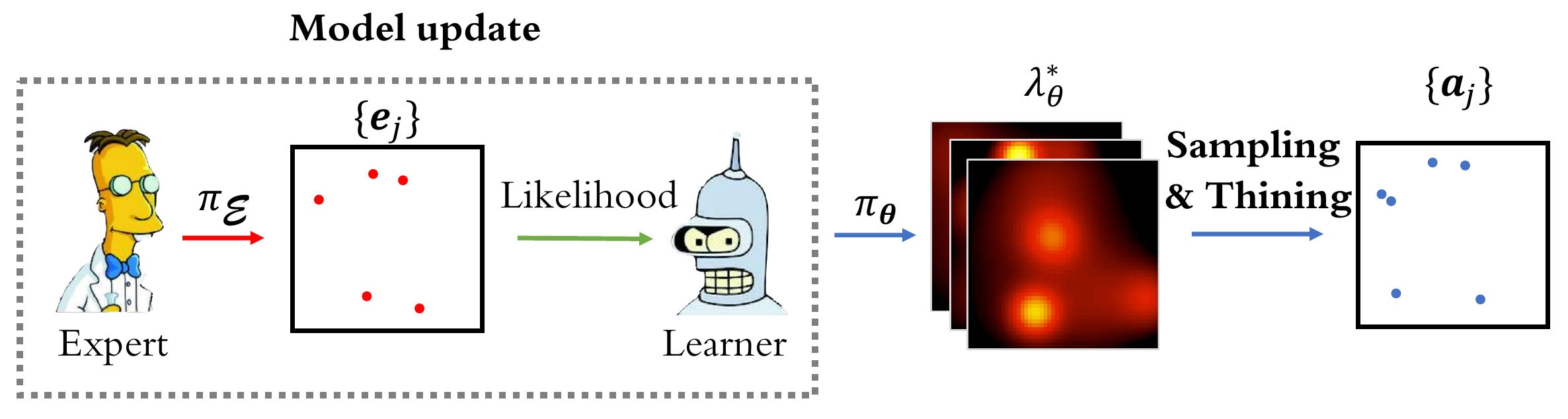}
\caption{Maximum likelihood estimation (MLE) for NEST}
\end{subfigure}
\vfill\vspace{.1in}
\begin{subfigure}[h]{1\linewidth}
\includegraphics[width=\linewidth]{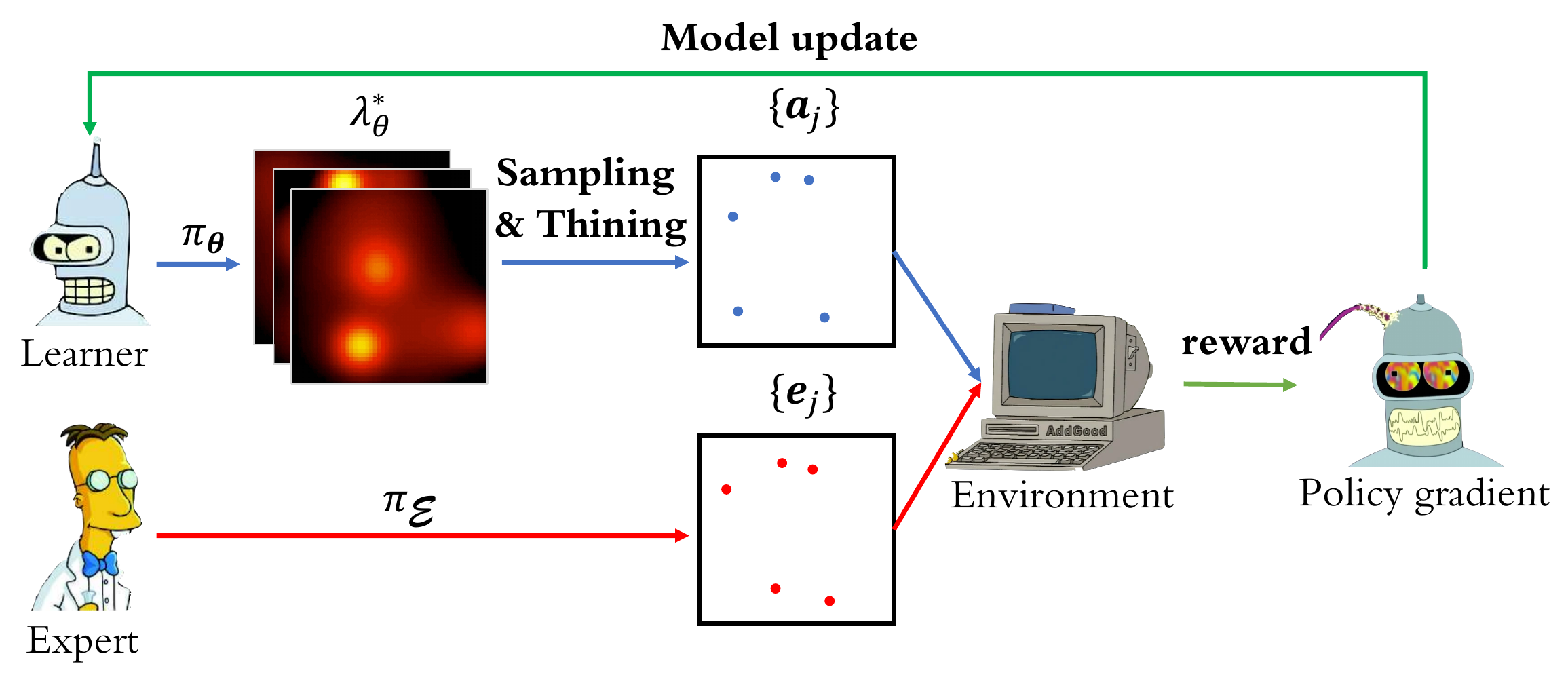}
\caption{Imitation learning (IL) for NEST}
\end{subfigure}
\caption{Comparison between the maximum likelihood and the imitation learning approaches. The main difference is that the MLE measure the ``likelihood'' of the data under a model. In contrast, our IL approach measures the actual divergence between the training data and the sequence generated from the model using MMD statistic without relying on model assumptions.}
\label{fig:rl-framework}
\end{figure}

The imitation learning framework is described as follows. Assume a learner takes actions $a := (t, s) \in [0, T) \times \mathcal{S}$ sequentially in an environment according to a specific {\it policy}, and the environment gives feedbacks (using a reward function) to the learner via observing the discrepancy between the learner actions and the demonstrations (training data) provided by an expert. In our setting, both the learner's actions and the demonstrations are over continuous-time and continuous space, which is a distinct feature of our problem. 

\subsubsection{Policy parameterization} 

Our desired learner policy is a probability density function of possible actions given the history $\mathcal H_t$. We define such function as $\pi(t, s): \mathbb [0, T) \times \mathcal{S} \rightarrow [0, 1]$, which assigns a probability to the next event at any possible location and time. Let the last event time before $T$ be $t_n$, and thus the next possible event is denoted as $(t_{n+1}, s_{n+1})$. The definition of the policy is
\[
\pi(t, s) = \mathbb P((t_{n+1}, s_{n+1}) \in [t, t+dt)\times B(s, \Delta s) |\mathcal H_t).
\]
 We will show that the policy can be explicitly related to the the conditional intensity function of the point process. 
\begin{lemma}\label{lemma1}
  A spatial temporal point process which generate new samples according to $\pi(t, s)$ has the corresponding conditional intensity function
  \begin{equation}
    \lambda^*_{\theta}(t, s) = \frac{\pi(t, s)}{1 - \int_{0}^{t}\int_{\mathcal{S}} \pi(\tau, r) d\tau dr}.
    \label{lambda_star}
  \end{equation}
\end{lemma}
From Lemma \ref{lemma1}, we can obtain the learner policy as the following proposition:
\begin{proposition}[Learner policy related to conditional intensity]\label{prop_policy}
  Given the conditional intensity function $\lambda^*_{\theta}(t, s)$, the learner policy of a STPP on $[0, T) \times \mathcal{S}$ is given by
 \begin{equation}
    \pi_{\theta}(t, s) =  \lambda^*_{\theta}(t, s) \cdot 
     \exp \left\{ - \int_{t_n}^{t} \int_{\mathcal{S}} \lambda^*_{\theta}(\tau, r) d\tau dr \right\}.
 \label{eq:cond_prob}
 \end{equation}
\end{proposition}

Thus, this naturally gives us a policy parameterization in a principled fashion: the learner policy $\pi_{\theta}$ is parameterized by $\theta$ based on the proposed heterogeneous Gaussian mixture diffusion kernel in Section~\ref{sec:model}. Note that using Proposition \ref{prop1}, which gives an explicit formula for the integral required in the exponent, the policy can be precisely evaluated even a deep neural network is included in the model.

\subsubsection{Imitation learning objective} 

Now assume the training data is generated by an expert policy $\pi_{E}$, where the subscript $E$ denotes ``expert''. 
Given a reward function $r(\cdot)$, the goal is to find an optimal policy that maximizes the expected reward
\[
\max_\theta J(\theta) :=  \mathbb{E}_{\boldsymbol{a} \sim \pi_{\theta}} \left[ \sum\nolimits_{i=1}^{n_a} r(a_i) \right], 
\]
where $\boldsymbol{a}=\{a_1, \dots, a_{n_{\alpha}}\}$ is one sampled roll-out from policy $\pi_{\theta}$. Note that $n_a$ can be different for different roll-out samples.

\subsubsection{Reward function} 

Consider the minimax formulation of imitation learning, which chooses the worst-case reward function that will give the maximum divergence between the rewards earned by the expert policy and the best learner policy:
\begin{equation*}
\underset{r \in \mathcal{F}}{\max}\Bigg ( \mathbb{E}_{\epsilon \sim \pi_{\mathcal{E}}} \left[ \sum_{i=1}^{n_e} r(e_i) \right] - 
 \underset{\pi_{\theta}\in\mathcal{G}}{\max}\ \mathbb{E}_{\alpha \sim \pi_{\theta}} \left [ \sum_{i=1}^{n_a} r(a_i) \right ] \Bigg ),
\label{eq:irl}
\end{equation*}
where $\mathcal{G}$ is the family of all candidate policies $\pi_{\theta}$ and $\mathcal{F}$ is the family class for reward function $r$ in reproducing kernel Hilbert space (RKHS).



We adopt a data-driven approach to solve the optimization problem and find the worst-case reward. This is related to inverse reinforcement learning; we borrow the idea of MMD reward in \cite{Kim2013, Gretton2007, Dziugaite2015, Li2018} and generalize it from a simple one-dimensional temporal point process to the more complex spatio-temporal setting. Suppose we are given training samples $\{\boldsymbol{e}_j\}$, $j=1,2,\dots,M_E$, which are the demonstrations provided by the expert $\pi_E$, where each $\boldsymbol{e}_j = \{e_0^{(j)}, e_1^{(j)}, \dots, e_{n_j}^{(j)}\}$ denotes a single demonstration. Also, we are given samples generated by the learner: let the trajectories generated by the learner $\pi_{\theta}$ denoted by $\{\boldsymbol{a}_i\}$, $i=1,2,\dots,M_L$, where each trajectory $\boldsymbol{a}_i = \{a_0^{(i)}, a_1^{(i)}, \dots, a_{n_i}^{(i)}\}$ denotes a single action trajectory. We will discuss how to generate samples in the following sub-section.

Using a similar argument as proving Theorem 1 in \cite{Li2018} and based on kernel embedding, we can obtain analytical expression for the worst-case reward function based on samples, which is given by 
\begin{equation}
\hat{r}(a) \propto 
 \frac{1}{M_E}\sum_{j=1}^{M_E} \sum_{u=1}^{n_j} k(e_{u}^{(j)}, a) - 
 \frac{1}{M_L}\sum_{i=1}^{M_L} \sum_{v=1}^{n_{i}} k(a_{v}^{(i)}, a).
\label{eq:reward-func}
\end{equation}
where $k(\cdot, \cdot)$ is a RKHS kernel. 
Here we use Gaussian kernel function, which achieves good experimental results on both synthetic data and real data. 

Using samples generated from learner policy, the gradient of $J(\theta)$ with respect to $\theta$ can be computed by using policy gradient with variance reduction \cite{Li2018},
\begin{equation*}
\nabla_\theta J(\theta) 
\approx \frac{1}{M_E} \sum_{j=1}^{M_E} \left[ \sum_{i=1}^{n_{j}} \left( \nabla_\theta \log \pi_{\theta}(a_i) \cdot \hat{r}(a_i) \right) \right].
\end{equation*}
The gradient of policy $\nabla_\theta \log \pi_{\theta}(a_i)$ can be computed analytically in closed-form since it is specified by the conditional intensity in Proposition \ref{prop_policy}, and the architecture of the neural network fully specifies the conditional intensity as we discussed in Section \ref{sec:model}.  

\begin{algorithm}[!t]
\SetAlgoLined
  {\bfseries input} $\theta, \lambda_0, \beta, T, \mathcal{S}$\;
  {\bfseries output} A set of events $\alpha$ ordered by time\;
  Initialize $\alpha = \emptyset$, $t=0$, $s \sim \texttt{uniform}(\mathcal{S})$, $s_n = 0$\;
  \While{$t < T$}{
    $u, D \sim \texttt{uniform}(0, 1)$; $s \sim \texttt{uniform}(\mathcal{S})$\;
    $\bar{\lambda} \leftarrow \lambda_0 + \sum_{(\tau, r) \in \alpha} \nu(t, \tau, s_n, r)$\; 
    $t \leftarrow t  - \ln u/\bar{\lambda}$\;
    Compute $\lambda^*_{\theta}(t, s)$ from (\ref{lambda_star})\;
    \If{$D \bar{\lambda} > \lambda^*_{\theta}(t, s)$}{
      $\alpha \leftarrow \alpha \cup \{(t, s)\}$; $s_n \leftarrow s$\;
    }
  }
\caption{Efficient thinning algorithm for STPP}
\label{algo:thinning}
\end{algorithm}

\begin{figure*}
\centering
\begin{subfigure}[h]{0.48\linewidth}
\includegraphics[width=\linewidth]{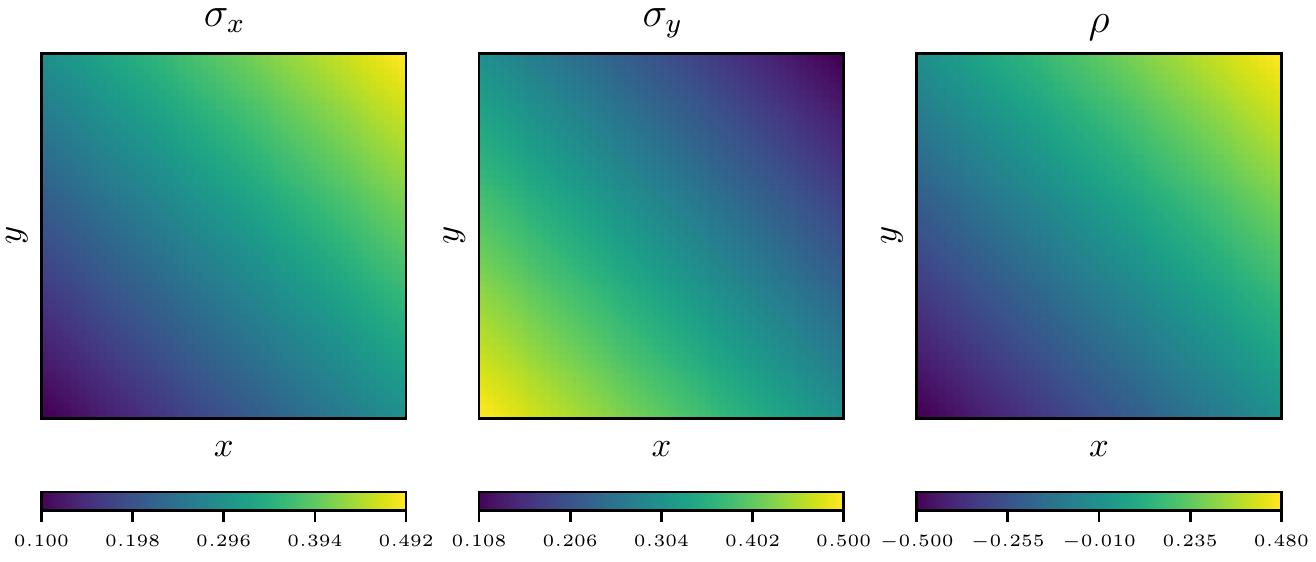}
\caption{ground truth for synthetic data set $1$}
\label{fig:exp-sim-1}
\end{subfigure}
\begin{subfigure}[h]{0.48\linewidth}
\includegraphics[width=\linewidth]{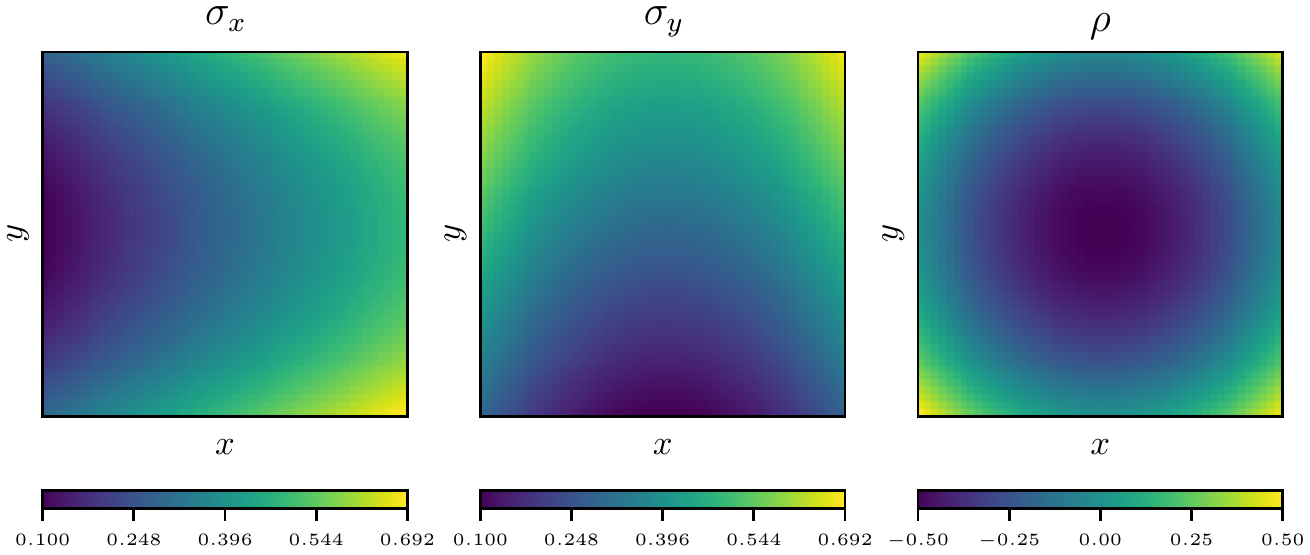}
\caption{ground truth for synthetic data set $2$}
\label{fig:exp-sim-2}
\end{subfigure}
    \vfill
\begin{subfigure}[h]{0.48\linewidth}
\includegraphics[width=\linewidth]{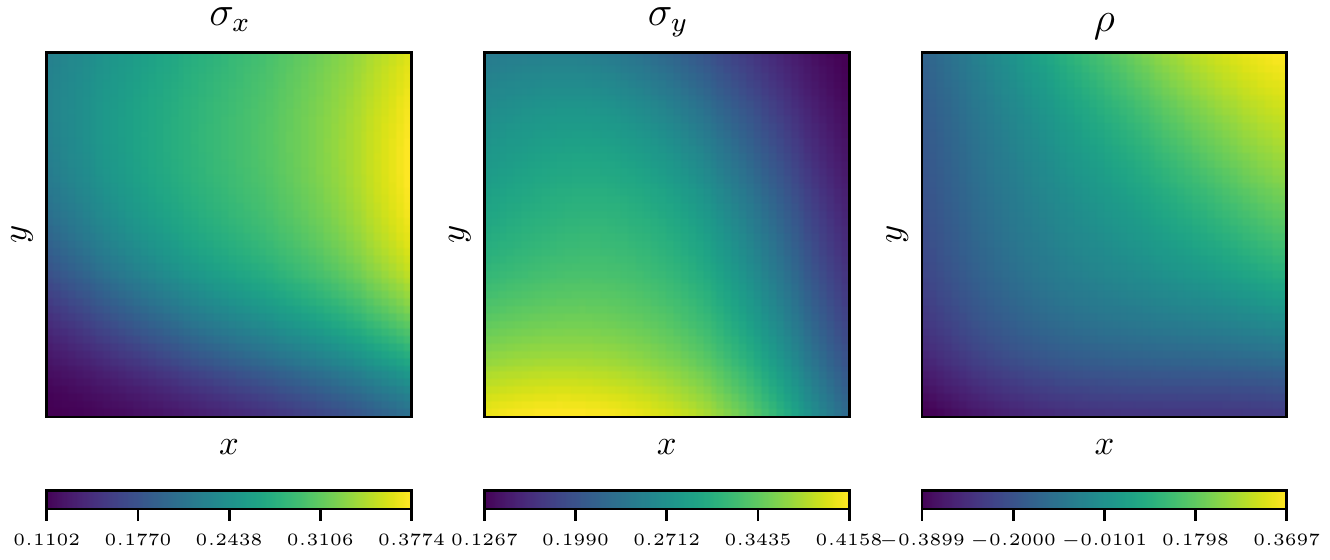}
\caption{recovered result for synthetic data set $1$}
\label{fig:exp-sim-1-res}
\end{subfigure}
\begin{subfigure}[h]{0.48\linewidth}
\includegraphics[width=\linewidth]{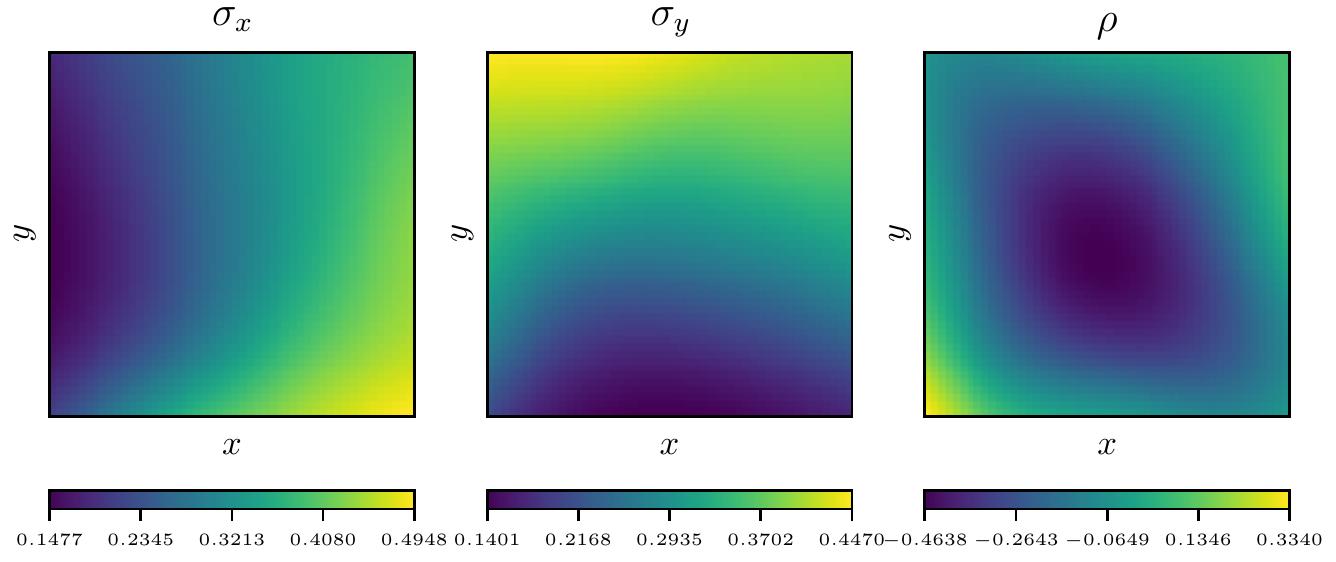}
\caption{recovered result for synthetic data set $2$}
\label{fig:exp-sim-2-res}
\end{subfigure}
\caption{Simulation results on two sets of synthetic data. (\ref{fig:exp-sim-1}): The ground truth of the Gaussian parameter $\sigma_x, \sigma_y, \rho$ in the synthetic data set $1$; (\ref{fig:exp-sim-2}): The ground truth of the Gaussian parameter $\sigma_x, \sigma_y, \rho$ in the synthetic data set $2$; (\ref{fig:exp-sim-1-res}): The recovered Gaussian parameters using the synthetic data set $1$; (\ref{fig:exp-sim-2-res}): The recovered Gaussian parameters using the synthetic data set $2$.}
\label{fig:sim-data-exp}
\end{figure*}


\subsubsection{Sampling from STPP using thinning algorithm} 

An important step in the imitation learning approach is to generate samples from our proposed model, i.e., $a \sim \pi_{\theta}$, given the history $\mathcal{H}_t$. Here, we develop an efficient sampling strategy to achieve good computational efficiency. We need to sample a point tuple $a=(t, s)$ according to the conditional intensity defined by (\ref{eq:con_intensity}). A default way to simulate point processes is to use the thinning algorithm \cite{Gabriel2013, Daley2008}. However, the vanilla thinning algorithm suffers from low sampling efficiency as it needs to sample in the space $|\mathcal{S}| \times [0, T)$ uniformly with the upper limit of the conditional intensity $\bar{\lambda}$ and only very few of candidate points will be retained in the end. In particular, given the parameter $\theta$, the procedure's computing complexity increases exponentially with the size of the sampling space. 
To improve sampling efficiency, we propose an efficient thinning algorithm summarized in Algorithm~\ref{algo:thinning}. The ``proposal'' density is a non-homogeneous STPP, whose intensity function is defined from the previous iterations. This analogous to the idea of rejection sampling \cite{Ogata1981}.


\section{Experiments}
\label{sec:experiments}

In this section, we evaluate our approaches by conducting experiments on both synthetic and real data sets. 

\begin{table}[!t]
\caption{Average MSE results on two synthetic data sets.}
\label{tab:mse-synthetic}
\vspace{-.1in}
\begin{center}
\begin{sc}
\resizebox{0.5\textwidth}{!}{%
\begin{tabular}{lcccccr}
\toprule[1pt]\midrule[0.3pt]
Data set & Random & ETAS & NEST+IL & NEST+MLE & RLPP \\
\midrule
\multicolumn{1}{c}{\begin{tabular}[c]{@{}c@{}}Synthetic 1 \\(space-time)\end{tabular}} & \multicolumn{1}{c}{\begin{tabular}[c]{@{}c@{}}.2781 \\ (.2324,.3153)\end{tabular}} & \multicolumn{1}{c}{\begin{tabular}[c]{@{}c@{}}.0433 \\ (.0402, .0489)\end{tabular}} & \multicolumn{1}{c}{\begin{tabular}[c]{@{}c@{}}\bf .0075 \\ (.0043, .0098)\end{tabular}} & \multicolumn{1}{c}{\begin{tabular}[c]{@{}c@{}}.0134 \\ (.0083, .0180)\end{tabular}} & N/A \\
\multicolumn{1}{c}{\begin{tabular}[c]{@{}c@{}}Synthetic 2 \\(space-time)\end{tabular}} & \multicolumn{1}{c}{\begin{tabular}[c]{@{}c@{}}.3327 \\ (.2831, .3793)\end{tabular}} & \multicolumn{1}{c}{\begin{tabular}[c]{@{}c@{}}.0512 \\ (.0490, .0543)\end{tabular}} & \multicolumn{1}{c}{\begin{tabular}[c]{@{}c@{}}\bf .0124 \\ (.0089, .0184)\end{tabular}} & \multicolumn{1}{c}{\begin{tabular}[c]{@{}c@{}}.0321 \\ (.0303, .0352)\end{tabular}} & N/A \\
\multicolumn{1}{c}{\begin{tabular}[c]{@{}c@{}}Synthetic 1 \\(time-only)\end{tabular}} & \multicolumn{1}{c}{\begin{tabular}[c]{@{}c@{}}.1734 \\ (.1495, .1936)\end{tabular}} & \multicolumn{1}{c}{\begin{tabular}[c]{@{}c@{}}.0135 \\ (.0113, .0149)\end{tabular}} & \multicolumn{1}{c}{\begin{tabular}[c]{@{}c@{}}\bf .0021 \\ (.0012, .0023)\end{tabular}} & \multicolumn{1}{c}{\begin{tabular}[c]{@{}c@{}}.0048 \\ (.0024, .0059)\end{tabular}} & \multicolumn{1}{c}{\begin{tabular}[c]{@{}c@{}}.0146 \\ (.0103, .0220)\end{tabular}} \\
\multicolumn{1}{c}{\begin{tabular}[c]{@{}c@{}}Synthetic 2 \\(time-only)\end{tabular}} & \multicolumn{1}{c}{\begin{tabular}[c]{@{}c@{}}.2147 \\ (.1812, .2455)\end{tabular}} & \multicolumn{1}{c}{\begin{tabular}[c]{@{}c@{}}.0323 \\ (.0301, .0357)\end{tabular}} & \multicolumn{1}{c}{\begin{tabular}[c]{@{}c@{}}\bf .0036 \\ (.0032, .0041)\end{tabular}} & \multicolumn{1}{c}{\begin{tabular}[c]{@{}c@{}}.0055 \\ (.0049, .0058)\end{tabular}} & \multicolumn{1}{c}{\begin{tabular}[c]{@{}c@{}}.0341 \\ (.0310, .0397)\end{tabular}} \\
\midrule[0.3pt]\bottomrule[1pt]
\end{tabular}
}
\end{sc}
\end{center}
\vspace{-.1in}
\end{table}

\begin{table}[!t]
\caption{Log-likelihood on two real data sets.}
\label{tab:loglik-comparison}
\vspace{-.1in}
\begin{center}
\begin{sc}
\resizebox{0.5\textwidth}{!}{%
\begin{tabular}{lcccccr}
\toprule[1pt]\midrule[0.3pt]
Data set & \multicolumn{1}{c}{ETAS} &  \multicolumn{1}{c}{\begin{tabular}[c]{@{}c@{}}NEST\\$K=1$\end{tabular}} & \multicolumn{1}{c}{\begin{tabular}[c]{@{}c@{}}NEST\\$K=5$\end{tabular}} & \multicolumn{1}{c}{\begin{tabular}[c]{@{}c@{}}NEST\\$K=10$\end{tabular}} & \multicolumn{1}{c}{\begin{tabular}[c]{@{}c@{}}NEST\\$K=20$\end{tabular}}\\
\midrule
Robbery &  \multicolumn{1}{c}{\begin{tabular}[c]{@{}c@{}}22.34\\(17.86, 25.12)\end{tabular}} & \multicolumn{1}{c}{\begin{tabular}[c]{@{}c@{}}28.43\\(26.13, 31.56)\end{tabular}} & \multicolumn{1}{c}{\begin{tabular}[c]{@{}c@{}}34.27\\(31.54, 36.53)\end{tabular}} & \multicolumn{1}{c}{\begin{tabular}[c]{@{}c@{}}35.10\\(32.48, 38.10)\end{tabular}} & \multicolumn{1}{c}{\begin{tabular}[c]{@{}c@{}}\bf 35.32\\(31.90, 39.34)\end{tabular}}\\
Seismic &  \multicolumn{1}{c}{\begin{tabular}[c]{@{}c@{}}138.1\\(129.3, 145.3)\end{tabular}} & \multicolumn{1}{c}{\begin{tabular}[c]{@{}c@{}}197.2\\(190.3, 206.4)\end{tabular}} & \multicolumn{1}{c}{\begin{tabular}[c]{@{}c@{}}224.1\\(219.1, 229.1)\end{tabular}} & \multicolumn{1}{c}{\begin{tabular}[c]{@{}c@{}}\bf 227.5\\(223.3, 230.5)\end{tabular}} & \multicolumn{1}{c}{\begin{tabular}[c]{@{}c@{}}226.9\\(222.2, 231.4)\end{tabular}} \\
\midrule[0.3pt]\bottomrule[1pt]
\end{tabular}
}
\end{sc}
\end{center}
\vspace{-.1in}
\end{table}

\begin{table*}[t!]
\caption{Average absolute MMD per sequence on two real data sets.}
\label{tab:mmd-comparison}
\vspace{-.1in}
\begin{center}
\begin{small}
\begin{sc}
\resizebox{.85\textwidth}{!}{%
\begin{tabular}{lccccccccr}
\toprule[1pt]\midrule[0.3pt]
Data set & \multicolumn{1}{c}{Random} & \multicolumn{1}{c}{ETAS} & \multicolumn{1}{c}{\begin{tabular}[c]{@{}c@{}}NEST+IL\\$K=5$\end{tabular}} & \multicolumn{1}{c}{\begin{tabular}[c]{@{}c@{}}NEST+MLE\\$K=1$\end{tabular}} & \multicolumn{1}{c}{\begin{tabular}[c]{@{}c@{}}NEST+MLE\\$K=5$\end{tabular}} & \multicolumn{1}{c}{\begin{tabular}[c]{@{}c@{}}NEST+MLE\\$K=10$\end{tabular}} & \multicolumn{1}{c}{\begin{tabular}[c]{@{}c@{}}NEST+MLE\\$K=20$\end{tabular}} & RLPP \\
\midrule
Robbery (space-time)    & \multicolumn{1}{c}{\begin{tabular}[c]{@{}c@{}}108.0\\(92.2, 125.4)\end{tabular}} & \multicolumn{1}{c}{\begin{tabular}[c]{@{}c@{}}72.9\\(66.1, 79.4)\end{tabular}} & \multicolumn{1}{c}{\begin{tabular}[c]{@{}c@{}}\textbf{68.4}\\(61.2, 72.6)\end{tabular}} & \multicolumn{1}{c}{\begin{tabular}[c]{@{}c@{}}71.6\\(64.2, 79.5)\end{tabular}} &  \multicolumn{1}{c}{\begin{tabular}[c]{@{}c@{}}69.4\\(64.7, 73.5)\end{tabular}} & \multicolumn{1}{c}{\begin{tabular}[c]{@{}c@{}}69.1\\(63.5, 74.1)\end{tabular}} & \multicolumn{1}{c}{\begin{tabular}[c]{@{}c@{}}68.9\\(61.1, 74.6)\end{tabular}} & N/A \\
Seismic (space-time)    & \multicolumn{1}{c}{\begin{tabular}[c]{@{}c@{}}53.1\\(40.0, 67.4)\end{tabular}} & \multicolumn{1}{c}{\begin{tabular}[c]{@{}c@{}}32.2\\(29.7, 36.4)\end{tabular}} & \multicolumn{1}{c}{\begin{tabular}[c]{@{}c@{}}\textbf{21.1}\\(19.5, 24.0)\end{tabular}} & \multicolumn{1}{c}{\begin{tabular}[c]{@{}c@{}}30.9\\(27.7, 32.8)\end{tabular}} & \multicolumn{1}{c}{\begin{tabular}[c]{@{}c@{}}28.6\\(24.3, 33.1)\end{tabular}} & \multicolumn{1}{c}{\begin{tabular}[c]{@{}c@{}}27.9\\(22.0, 34.3)\end{tabular}} & \multicolumn{1}{c}{\begin{tabular}[c]{@{}c@{}}28.3\\(23.8, 35.7)\end{tabular}} & N/A \\
Robbery (time only)     & \multicolumn{1}{c}{\begin{tabular}[c]{@{}c@{}}126.8\\(109.3, 150.1)\end{tabular}} & \multicolumn{1}{c}{\begin{tabular}[c]{@{}c@{}}90.6\\(85.2, 97.3)\end{tabular}} & \multicolumn{1}{c}{\begin{tabular}[c]{@{}c@{}}\textbf{82.2}\\(74.5, 89.0)\end{tabular}} & \multicolumn{1}{c}{\begin{tabular}[c]{@{}c@{}}88.1\\(80.4, 95.1)\end{tabular}} & \multicolumn{1}{c}{\begin{tabular}[c]{@{}c@{}}83.8\\(77.3, 87.9)\end{tabular}} & \multicolumn{1}{c}{\begin{tabular}[c]{@{}c@{}}83.1\\(76.0, 86.4)\end{tabular}} & \multicolumn{1}{c}{\begin{tabular}[c]{@{}c@{}}83.0\\(76.2, 84.5)\end{tabular}} & \multicolumn{1}{c}{\begin{tabular}[c]{@{}c@{}}83.1\\(75.1, 89.2)\end{tabular}} \\
Seismic (time only)     & \multicolumn{1}{c}{\begin{tabular}[c]{@{}c@{}}60.7\\(55.1, 66.9)\end{tabular}} & \multicolumn{1}{c}{\begin{tabular}[c]{@{}c@{}}51.0\\(48.5, 53.1)\end{tabular}} & \multicolumn{1}{c}{\begin{tabular}[c]{@{}c@{}}\textbf{21.2}\\(18.8, 24.1)\end{tabular}} & \multicolumn{1}{c}{\begin{tabular}[c]{@{}c@{}}35.3\\(29.7, 39.2)\end{tabular}} &  \multicolumn{1}{c}{\begin{tabular}[c]{@{}c@{}}27.3\\(24.6, 30.4) \end{tabular}} & \multicolumn{1}{c}{\begin{tabular}[c]{@{}c@{}}25.1\\(22.0, 28.9)\end{tabular}} & \multicolumn{1}{c}{\begin{tabular}[c]{@{}c@{}}24.6\\(21.6, 27.5)\end{tabular}} & \multicolumn{1}{c}{\begin{tabular}[c]{@{}c@{}}23.3\\(20.3, 28.2)\end{tabular}} \\
\midrule[0.3pt]\bottomrule[1pt]
\end{tabular}
}
\end{sc}
\end{small}
\end{center}
\end{table*}

\begin{table*}[t!]
\caption{Average MSE results on two real data sets.}
\label{tab:mse-real}
\vspace{-.1in}
\begin{center}
\begin{small}
\begin{sc}
\resizebox{.85\textwidth}{!}{%
\begin{tabular}{lccccccccr}
\toprule[1pt]\midrule[0.3pt]
Data set & \multicolumn{1}{c}{Random} & \multicolumn{1}{c}{ETAS} & \multicolumn{1}{c}{\begin{tabular}[c]{@{}c@{}}NEST+IL\\$K=5$\end{tabular}} & \multicolumn{1}{c}{\begin{tabular}[c]{@{}c@{}}NEST+MLE\\$K=1$\end{tabular}} & \multicolumn{1}{c}{\begin{tabular}[c]{@{}c@{}}NEST+MLE\\$K=5$\end{tabular}} & \multicolumn{1}{c}{\begin{tabular}[c]{@{}c@{}}NEST+MLE\\$K=10$\end{tabular}} & \multicolumn{1}{c}{\begin{tabular}[c]{@{}c@{}}NEST+MLE\\$K=20$\end{tabular}} & RLPP \\
\midrule
Robbery (space-time)    &  \multicolumn{1}{c}{\begin{tabular}[c]{@{}c@{}}.6323\\(.6143, .6583)\end{tabular}} & \multicolumn{1}{c}{\begin{tabular}[c]{@{}c@{}}.1425\\(.1368, .1497)\end{tabular}} & \multicolumn{1}{c}{\begin{tabular}[c]{@{}c@{}}\textbf{.0503}\\(.0445, .0549)\end{tabular}} & \multicolumn{1}{c}{\begin{tabular}[c]{@{}c@{}}.1144\\(.1078, .1321)\end{tabular}} & \multicolumn{1}{c}{\begin{tabular}[c]{@{}c@{}}.0649\\(.0574, .0693)\end{tabular}} & \multicolumn{1}{c}{\begin{tabular}[c]{@{}c@{}}.0610\\(.0544, .0679)\end{tabular}} & \multicolumn{1}{c}{\begin{tabular}[c]{@{}c@{}}.0601\\(.0531, .0672)\end{tabular}} & N/A \\
Seismic (space-time)    & \multicolumn{1}{c}{\begin{tabular}[c]{@{}c@{}}.2645\\(.2457, .2833)\end{tabular}} & \multicolumn{1}{c}{\begin{tabular}[c]{@{}c@{}}.0221\\(.0180, .0277)\end{tabular}} & \multicolumn{1}{c}{\begin{tabular}[c]{@{}c@{}}\textbf{.0119}\\(.0092, .0134)\end{tabular}} & 
\multicolumn{1}{c}{\begin{tabular}[c]{@{}c@{}}.0203\\(.0173, .0252)\end{tabular}} & \multicolumn{1}{c}{\begin{tabular}[c]{@{}c@{}}.0153\\(.0133, .0181)\end{tabular}} & \multicolumn{1}{c}{\begin{tabular}[c]{@{}c@{}}.0142\\(.0132, .0174)\end{tabular}} & \multicolumn{1}{c}{\begin{tabular}[c]{@{}c@{}}.0143\\(.0128, .0182)\end{tabular}} & N/A \\
Robbery (time only)     & \multicolumn{1}{c}{\begin{tabular}[c]{@{}c@{}}.4783\\(.4553, .4940)\end{tabular}} & \multicolumn{1}{c}{\begin{tabular}[c]{@{}c@{}}.0857\\(.0788, .0914)\end{tabular}} & \multicolumn{1}{c}{\begin{tabular}[c]{@{}c@{}}.0104\\(.0082, .0141)\end{tabular}} & \multicolumn{1}{c}{\begin{tabular}[c]{@{}c@{}}.0583\\(.0492, .0651)\end{tabular}} & \multicolumn{1}{c}{\begin{tabular}[c]{@{}c@{}}.0094\\(.0064, .0134)\end{tabular}} & \multicolumn{1}{c}{\begin{tabular}[c]{@{}c@{}}\bf .0089\\(.0052, .0123)\end{tabular}} & \multicolumn{1}{c}{\begin{tabular}[c]{@{}c@{}}.0091\\(.0056, .0143)\end{tabular}} & \multicolumn{1}{c}{\begin{tabular}[c]{@{}c@{}}.0183\\(.0126, .0242)\end{tabular}} \\
Seismic (time only)     & \multicolumn{1}{c}{\begin{tabular}[c]{@{}c@{}}.1266\\(.1001, .1562)\end{tabular}} & \multicolumn{1}{c}{\begin{tabular}[c]{@{}c@{}}.0173\\(.0134, .0210)\end{tabular}} & \multicolumn{1}{c}{\begin{tabular}[c]{@{}c@{}}\textbf{.0045}\\(.0036, .0078)\end{tabular}} & \multicolumn{1}{c}{\begin{tabular}[c]{@{}c@{}}.0175\\(.0142, .0211)\end{tabular}} & \multicolumn{1}{c}{\begin{tabular}[c]{@{}c@{}}.0150\\(.0131, .0198)\end{tabular}} & \multicolumn{1}{c}{\begin{tabular}[c]{@{}c@{}}.0143\\(.0124, .0177)\end{tabular}} & \multicolumn{1}{c}{\begin{tabular}[c]{@{}c@{}}.0144\\(.0122, .0183)\end{tabular}} & \multicolumn{1}{c}{\begin{tabular}[c]{@{}c@{}}.0122\\(.0112, .0154)\end{tabular}} \\
\midrule[0.3pt]\bottomrule[1pt]
\end{tabular}
}
\end{sc}
\end{small}
\end{center}
\end{table*}

\subsection{Baselines and evaluation metrics}

This section compares our proposed neural embedding spatio-temporal (\texttt{NEST}) with a benchmark and several state-of-the-art methods in the field. These include (1) \texttt{Random uniform} that randomly makes actions in the action space; (2) epidemic type aftershock-sequences (\texttt{ETAS}) with standard diffusion kernel, which is currently the most widely used approach in spatio-temporal event data modeling. For ETAS, the parameters are estimated by maximum likelihood estimate (MLE); (3) reinforcement learning point processes model (\texttt{RLPP}) \cite{Li2018} is for modeling temporal point process only, which cannot be easily generalized to spatio-temporal models. (4) \texttt{NEST+IL} is our approach using imitation learning;  (5) \texttt{NEST+MLE} is our approach where the parameters are estimated by MLE.
We also investigate our method's performance with a different number of Gaussian components ($K$) under the same experimental settings for the real data.

To evaluate the performance of algorithms (i.e., various generative models), we adopt two performance metrics: 
(1) the average mean square error (MSE) of the {\it one-step-ahead prediction}. 
The {\it one-step-ahead prediction} for the next event $(t_{n+1}, s_{n+1})$ is carried out by computing the expectation of the conditional probability function (or policy) defined in \eqref{eq:cond_prob} given the past observation $\mathcal{H}_{t_{n+1}}$, i.e.,
\[
    \begin{bmatrix}
        \hat{t}_{n+1} \\
        \hat{s}_{n+1}
    \end{bmatrix} = 
    \begin{bmatrix}
        \int_{t_{n}}^{T} \tau \int_\mathcal{S} \pi_\theta(\tau, \omega) d\omega d\tau \\
        \int_\mathcal{S} \omega \int_{t_{n}}^{T} \pi_\theta(\tau, k) d\tau d\omega
    \end{bmatrix}.
\]
Due to the close-form expression of $\pi_\theta$ in (\ref{eq:cond_prob}), the integration above can also be obtained analytically. 
(2) the maximum mean discrepancy (MMD) metric between the real observed sequences and the generated sequences from the models, as specified in \eqref{eq:reward-func}.
To obtain and compare the confidence interval of the performance metrics, we repeat each experiment 20 times.
For synthetic data, we also directly compare the recovered parameters against the true parameters used to construct the generator and generate the corresponding synthetic data. 

\subsection{Synthetic data}
\label{sec:simulation}


To validate the capability of the NEST in capturing the spatial information from discrete event data,
we first evaluate our method on two synthetic data sets. 
As an ablation study, these two data sets are generated by two \texttt{NEST} with a single Gaussian component ($K=1$), where their artificial parameters (ground truth) have been shown in Fig.~\ref{fig:exp-sim-1} and \ref{fig:exp-sim-2}, respectively. 
The parameters in synthetic data set 1 (Fig.~\ref{fig:exp-sim-1}) are linearly related to their spatial locations and parameters in synthetic data set 2 (Fig.~\ref{fig:exp-sim-2}) are non-linearly related to their spatial locations. 
In both data sets, there are 5,000 sequences with an average length of 191, and 80\%, 20\% sequences are used for training and testing. The time and location of events have been normalized to $T=10$ and $\mathcal{S} = [-1, +1] \times [-1, +1]$. We specify three hidden layers with 64 nodes per layer for the neural network defined by $\theta_h$. The optimizer randomly takes 40 sequences as a batch during the training. Both models are trained by both MLE and IL approaches. 


We fit our \texttt{NEST} models using these two synthetic data sets separately and obtain the corresponding parameters.
Fig.~\ref{fig:exp-sim-1-res} and Fig.~\ref{fig:exp-sim-2-res} show the recovered parameters learned from two synthetic data sets. It shows that the spatial distribution of the recovered parameters resembles the true parameters. It also confirms that our model can capture the underlying spatial information from both linear and non-linear parameter space. 
In Table~\ref{tab:mse-synthetic}, we report the MSE of the one-step-ahead prediction and the 25\% and 75\% quantile of the performance metric repeated over 20 experiments (in brackets). Clearly, our models (\texttt{NEST+ML} and \texttt{NEXT+IL}) consistently outperform other methods on two synthetic data sets.

\subsection{Real data}
\label{sec:real}


Now we test our approaches on two real-world data sets: Atlanta 911 calls-for-services data (provide by the Atlanta Police Department to us under a data agreement) and Northern California seismic data \cite{NCEDC2014}. For the ease of comparison, we normalize the space region of both data sets to the same space $T \times \mathcal{S}$ where $T = (0, 10]$ and $\mathcal{S} = [-1, 1] \times [-1, 1]$. The detailed description of the two data sets is as follows.

\vspace{.05in}
\noindent\textbf{Atlanta 911 calls-for-service data.} The Atlanta Police Department provides the 911 calls-for-service data in Atlanta from the end of 2015 to 2017 (this data set has been previously used to validate the crime linkage detection algorithm \cite{Zhu2018-1, Zhu2018-2, Zhu2019A}). We extract 7,831 reported robbery from the data set since robbers usually follow particular \textit{modus operandi} (M.O.), where criminal spots and times tend to have a causal relationship with each other. Each robbery report is associated with the time (accurate to the second) and the geolocation (latitude and longitude) indicating when and where the robbery occurred. We consider each series of robbery as a sequence.

\vspace{.05in}
\noindent \textbf{Northern California seismic data.} The Northern California Earthquake Data Center (NCEDC) provides public time series data \cite{NCEDC2014} that comes from broadband, short period, strong motion seismic sensors, GPS, and other geophysical sensors. We extract 16,401 seismic records with a magnitude larger than 3.0 from 1978 to 2018 in Northern California and partition the data into multiple sequences every quarter. To test our model, we only utilize time and geolocation in the record. 


\begin{figure*}[!t]
\centering
\begin{subfigure}[h]{0.245\linewidth}
\includegraphics[width=\linewidth]{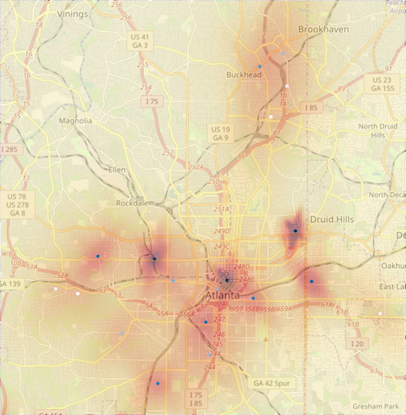}
\caption{Robbery seq at $t_1$ (NEST)}
\label{fig:exp-robbery-mle-2}
\end{subfigure}
\begin{subfigure}[h]{0.245\linewidth}
\includegraphics[width=\linewidth]{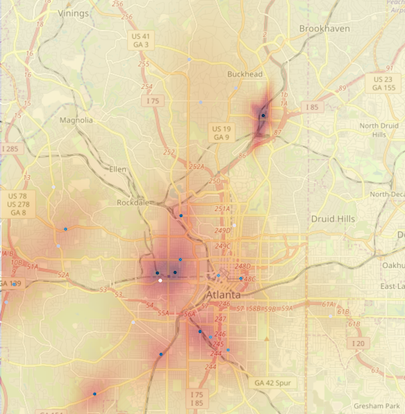}
\caption{Robbery seq at $t_2$ (NEST)}
\label{fig:exp-robbery-mle-1}
\end{subfigure}
\begin{subfigure}[h]{0.245\linewidth}
\includegraphics[width=\linewidth]{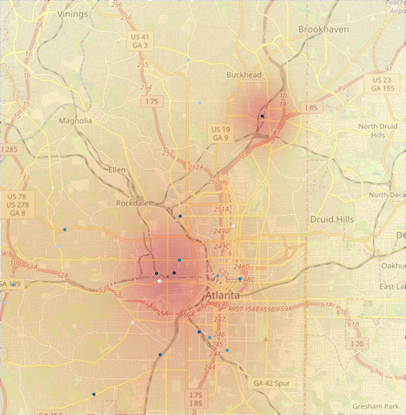}
\caption{Robbery seq at $t_2$ (ETAS)}
\label{fig:exp-robbery-soa-1}
\end{subfigure}
    \vfill
\begin{subfigure}[h]{0.245\linewidth}
\includegraphics[width=\linewidth]{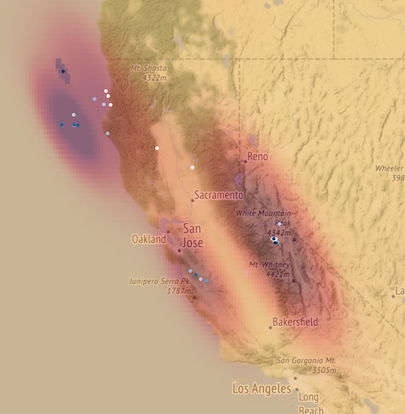}
\caption{Seismic seq at $t_1$ (NEST)}
\label{fig:exp-earthquake-mle-2}
\end{subfigure}
\begin{subfigure}[h]{0.245\linewidth}
\includegraphics[width=\linewidth]{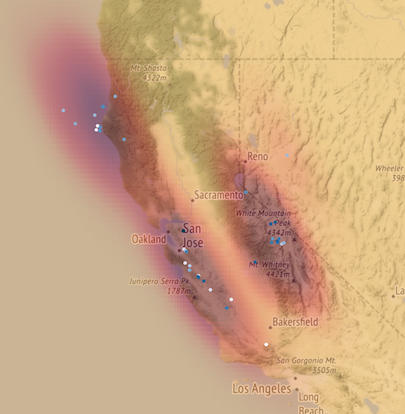}
\caption{Seismic seq at $t_2$ (NEST)}
\label{fig:exp-earthquake-mle-1}
\end{subfigure}
\begin{subfigure}[h]{0.245\linewidth}
\includegraphics[width=\linewidth]{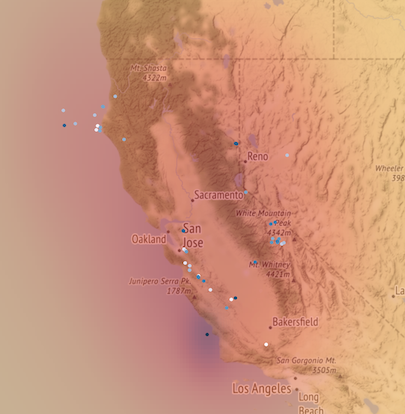}
\caption{Seismic seq at $t_2$ (ETAS)}
\label{fig:exp-earthquake-soa-1}
\end{subfigure}
\caption{Snapshots of the conditional intensity for two real data sequences (crime events in Atlanta and seismic events): (\ref{fig:exp-robbery-mle-2}, \ref{fig:exp-robbery-mle-1}, \ref{fig:exp-robbery-soa-1}): Snapshots of the conditional intensity for a series of robberies in Atlanta. (\ref{fig:exp-earthquake-mle-2}, \ref{fig:exp-earthquake-mle-1}, \ref{fig:exp-earthquake-soa-1}): Snapshots of the conditional intensity for a series of earthquakes in Northern California. (\ref{fig:exp-robbery-mle-1}, \ref{fig:exp-robbery-mle-2}, \ref{fig:exp-earthquake-mle-1}, \ref{fig:exp-earthquake-mle-2}) are generated by \texttt{NEST+IL} ($K=5$) and (\ref{fig:exp-earthquake-soa-1}, \ref{fig:exp-robbery-soa-1}) are generated by \texttt{ETAS}, respectively. The color depth indicates the value of intensity. The region in darker red has a higher risk to have the next event happened again.
The dots on the maps represent the occurred events in which dark blue represents the newly happened events, and white represents events that happened initially.}
\label{fig:real-data-exp}
\vspace{.1in}
\end{figure*}

\begin{figure*}[!t]
\centering
\begin{subfigure}{0.245\linewidth}
\includegraphics[width=\linewidth]{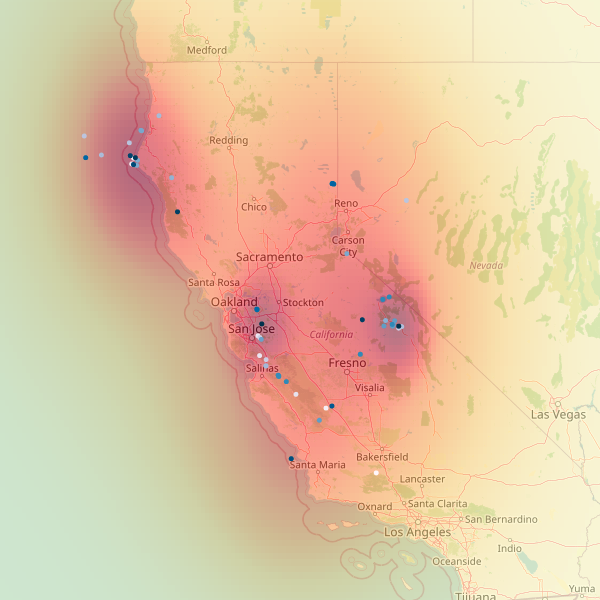}
\caption{$K=1$}
\end{subfigure}
\begin{subfigure}{0.245\linewidth}
\includegraphics[width=\linewidth]{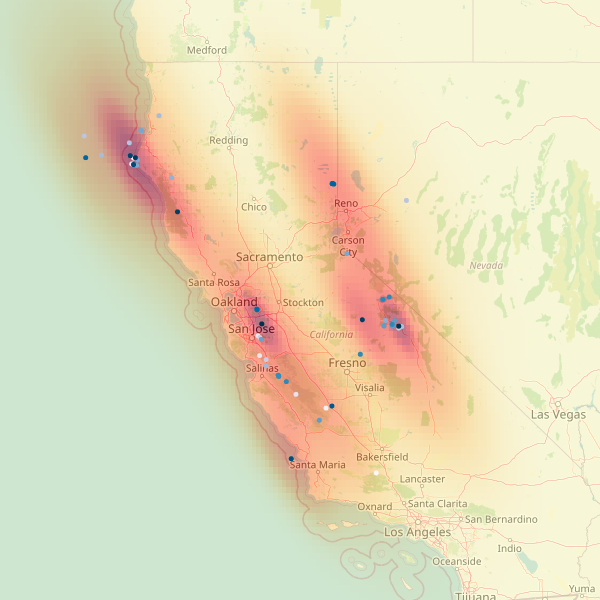}
\caption{$K=5$}
\end{subfigure}
\begin{subfigure}{0.245\linewidth}
\includegraphics[width=\linewidth]{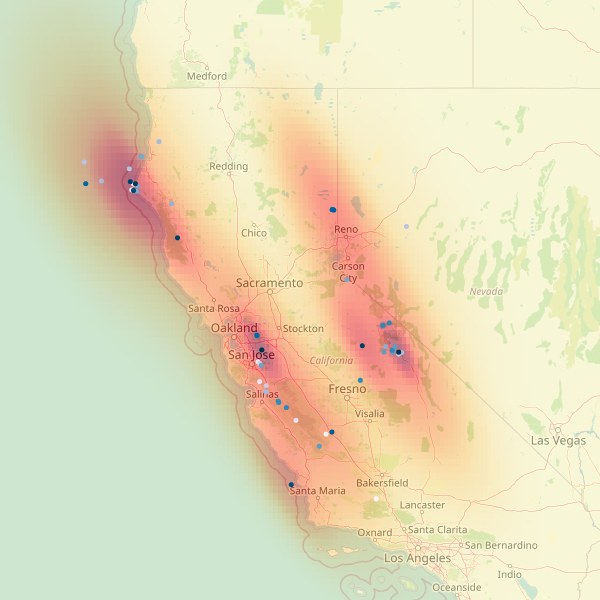}
\caption{$K=10$}
\end{subfigure}
\begin{subfigure}{0.245\linewidth}
\includegraphics[width=\linewidth]{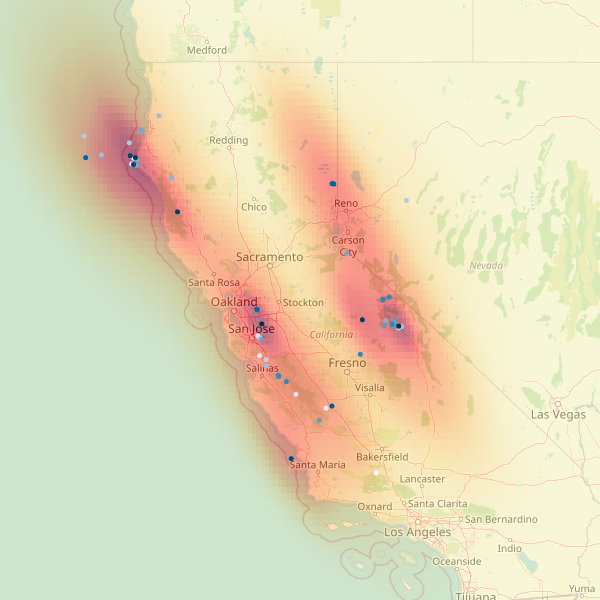}
\caption{$K=20$}
\end{subfigure}
\caption{Illustration of the conditional intensity function for seismic events in Northern California over different the number of components of the Gaussian mixture $K$, while other experimental settings are the same. As $K$ increases from 1 to 5, the model can capture more details in conditional spatial intensities; also, when $K \ge 5$, there seems to be little difference.}
\label{fig:northcal-diff-comp}
\end{figure*}

\begin{figure*}[!t]
\centering
\begin{subfigure}{0.245\linewidth}
\includegraphics[width=\linewidth]{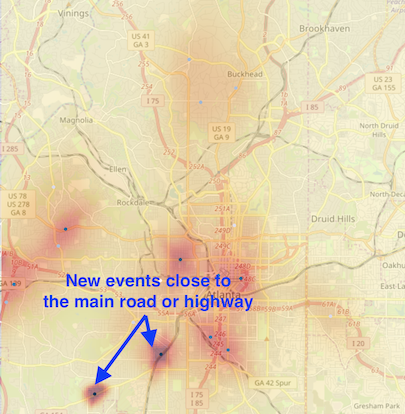}
\caption{$t_1 \approx 5.6$}
\label{fig:exp-robbery-t1}
\end{subfigure}
\begin{subfigure}{0.245\linewidth}
\includegraphics[width=\linewidth]{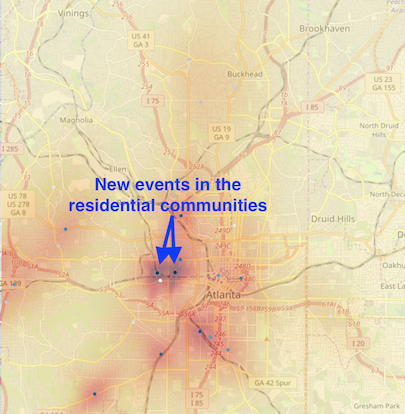}
\caption{$t_2 \approx 6.8$}
\label{fig:exp-robbery-t2}
\end{subfigure}
\begin{subfigure}{0.245\linewidth}
\includegraphics[width=\linewidth]{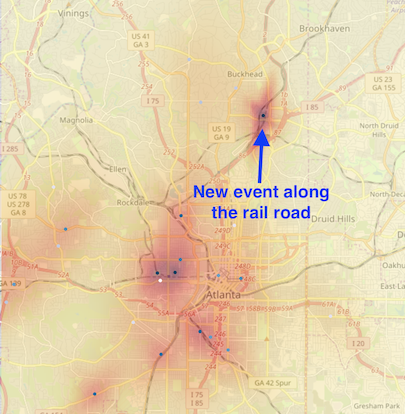}
\caption{$t_3 \approx 7.0$}
\label{fig:exp-robbery-t3}
\end{subfigure}
\begin{subfigure}{0.245\linewidth}
\includegraphics[width=\linewidth]{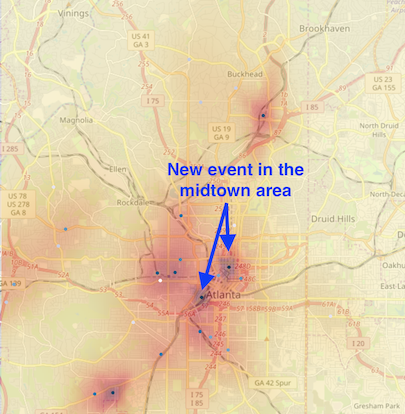}
\caption{$t_4 \approx 7.5$}
\label{fig:exp-robbery-t4}
\end{subfigure}
\caption{Snapshots of the conditional intensity of a robbery sequence at four time points. The color depth indicates the value of intensity. The region in darker red has a higher risk to have robbery event happened again. It also shows that robbery events in different regions have very distinctive diffusion patterns, which may be related to the city's geographic and demographic features.}
\label{fig:real-time-series-exp}
\end{figure*}

\vspace{.05in}
We first quantitatively compare our \texttt{NEST+MLE} and \texttt{ETAS} by evaluating the obtained log-likelihood per sequence for both real-world data sets. As shown in Table~\ref{tab:loglik-comparison}, our model attains a larger log-likelihood value on both data sets comparing to the state-of-the-art (\texttt{ETAS}), which confirms that our model can obtain a better fit on these two real data sets. 
We also compare \texttt{NEST+MLE} and \texttt{NEST+IL} with other baseline methods by measuring their average absolute MMD values and average MSE of the one-step-ahead prediction. The average absolute MMD value can be measured by computing the pairwise distances between the observed training data and the generated sequences according to \eqref{eq:reward-func}. The absolute MMD value indicates the level of similarity of two arbitrary distributions. If two distributions are the same, then their absolute MMD value is zero. In the experiment, we randomly pick 100 pairs of generated sequences against observed sequences from the real data sets and compute their average absolute MMD value. 
We summarized the results of the average absolute MDD and average MSE and the 25\% and 75\% quantile of performance metrics repeated over 20 experiments (in brackets) in Table~\ref{tab:mmd-comparison} and Table~\ref{tab:mse-real}, respectively.
The results show that both our methods \texttt{NEST+IL} and \texttt{NEST+MLE} outperform the state-of-the-art (\texttt{ETAS}) regarding both metrics. 
We also show that our method has a competitive performance even without considering spatial information comparing to \texttt{RLPP}.
From Fig.~\ref{fig:northcal-diff-comp} and Table~\ref{tab:loglik-comparison}, \ref{tab:mmd-comparison}, \ref{tab:mse-real}, we find that increasing the number of Gaussian component can improve the model performance. However, when $K \ge 5$, there is little difference (this may vary case by case). 

%
To interpret the spatial dependence learned using our model, we visualize the spatial conditional intensity over space at a specific time frame. For the state-of-the-art, as shown in Fig.~\ref{fig:exp-robbery-soa-1}, \ref{fig:exp-earthquake-soa-1}, we can see that \texttt{ETAS} captured the general pattern of the conditional intensity over space, where regions with more events tend to have higher intensity. Comparing to the results shown in Fig.~\ref{fig:exp-robbery-mle-2}, \ref{fig:exp-robbery-mle-1}, \ref{fig:exp-earthquake-mle-2}, \ref{fig:exp-earthquake-mle-1}, our \texttt{NEST} is able to capture more intricate spatial pattern at different locations and the shape of the captured diffusion also varies depending locations. The results clearly demonstrates that the Pacific-North America plate motion is along two parallel fault systems -- the San Andreas fault, and Eastern California shear zone and Walker Lane.

For 911 calls-for-service data, shown in Fig.~\ref{fig:exp-robbery-mle-2}, \ref{fig:exp-robbery-mle-1}, the spatial influence of some robbery events diffuse to the surrounding streets and the community blocks unevenly. For seismic data, shown in Fig.~\ref{fig:exp-earthquake-mle-1}, \ref{fig:exp-earthquake-mle-2}, the spatial influence of some events majorly diffuses along the earthquake fault lines. 
We also visualize the fitted parameters of each Gaussian component of the NEST models on both data sets, as shown in Fig.~\ref{fig:param-earthquake} and Fig.~\ref{fig:param-robbery} in the appendix.


Besides, Fig.~\ref{fig:real-time-series-exp} shows a series of snapshots of the conditional intensity for a sequence of robberies at different time points. Each subfigure shows that events in different urban regions have various shapes of spatial influence on their neighborhoods. 
For example, events near the highway or railroad tend to spread their influence along the road. In contrast, events around the Midtown area tend to spread their influence more evenly to the neighboring communities. 

\subsection{Learning efficiency} 

\begin{table}[!t]
\caption{Training time per batch on two real data sets.}
\label{tab:time-per-batch}
\vspace{-.1in}
\begin{center}
\begin{sc}
\resizebox{0.5\textwidth}{!}{%
\begin{tabular}{lccccccr}
\toprule[1pt]\midrule[0.3pt]
Data set & \multicolumn{1}{c}{ETAS} & \multicolumn{1}{c}{\begin{tabular}[c]{@{}c@{}}NEST+IL\\$K=5$\end{tabular}} & \multicolumn{1}{c}{\begin{tabular}[c]{@{}c@{}}NEST+MLE\\$K=1$\end{tabular}} & \multicolumn{1}{c}{\begin{tabular}[c]{@{}c@{}}NEST+MLE\\$K=5$\end{tabular}} & \multicolumn{1}{c}{\begin{tabular}[c]{@{}c@{}}NEST+MLE\\$K=10$\end{tabular}} & \multicolumn{1}{c}{\begin{tabular}[c]{@{}c@{}}NEST+MLE\\$K=20$\end{tabular}} & RLPP\\
\midrule
\multicolumn{1}{c}{\begin{tabular}[c]{@{}c@{}}Robbery\\(space-time) \end{tabular}} &  12s & 312s & 34s & 43s & 60s & 94s  & N / A \\
\multicolumn{1}{c}{\begin{tabular}[c]{@{}c@{}}Seismic\\(space-time) \end{tabular}} &  22s & 457s & 42s & 69s & 91s & 110s & N / A \\
\multicolumn{1}{c}{\begin{tabular}[c]{@{}c@{}}Robbery\\(time only) \end{tabular}} &  3s & 142s & 9s & 16s & 25s & 34s & 160s\\
\multicolumn{1}{c}{\begin{tabular}[c]{@{}c@{}}Seismic\\(time only) \end{tabular}} &  3s & 194s & 12s & 21s & 34s & 47s & 216s\\ 
\midrule[0.3pt]\bottomrule[1pt]
\end{tabular}
}
\end{sc}
\end{center}
\vspace{-.1in}
\end{table}

To compare the efficiencies between different learning strategies, we perform the experiments on the same data sets given a single laptop's computational power with quad-core processors which speed up to 4.7 GHz, and record the training time per batch accordingly. In Table~\ref{tab:time-per-batch}, we show that the training process of \texttt{NEST+MLE} is faster in general than \texttt{NEST+IL}; in the imitation learning framework, events generation is time-consuming (even using our proposed efficient thinning algorithm) compared to the learning process. 
Also, considering more Gaussian components in the NEST model can also increase the training time significantly. 
Comparing to other baseline methods \cite{Xiao2017B, Li2018}, including \texttt{RLPP}, which are usually over-parameterized on simple problems, our \texttt{NEST} can achieve better results with less training time using a fewer number of parameters since the deep neural network in \texttt{NEST} is only used to capture the non-linear spatial pattern in the embedding space.

\section{Conclusion and discussions}
\label{sec:conclusion}

We have presented a generative model for spatio-temporal event data by designing heterogeneous and flexible spatio-temporal point process models whose parameters are parameterized by neural networks, representing complex-shaped spatial dependence. Unlike a full neural network approach, our parameterization based on Gaussian diffusion kernels can still enjoy interpretability: the shape of the spatial-temporal influence kernels, which are useful for deriving knowledge for domain experts such as seismic data analysis. Also, we present a flexible model-fitting approach based on imitation learning, a form of reinforcement learning, which does not rely on exact model specification since it is based on directly measuring the divergence between the empirical distributions of the training data and the data generated from data (in contrast to MLE, which measures the ``likelihood'' of data under a pre-specified model structure). The robustness of the imitation learning-based training approach for the NEST model is demonstrated using both synthetic and real data sets, in terms of the predictive capability (measured by MSE of the conditional intensity function, which measure the probability for an event to occur at a location in the future given the past), and the maximum-mean-divergence (measures how realistic the data generated from the model matches to the real training data). We have observed from synthetic data and real-data that the imitation learning approach will achieve a significant performance gain than maximum likelihood and more robust for training \texttt{NEST} model.  Our \texttt{NEST} can have a wide range of applications: it is a probabilistic generative model and can be used to evaluate the likelihood of a sequence and perform anomaly detection spatio-temporal event data, as well as making predictions about future incidents given the history information.

Future work may include extending to marked spatio-temporal processes, for instance, to consider the magnitude of the earthquake in the models in the current set-up. Moreover, it will be interesting to consider how to incorporate {\it known} spatial structural information in the spatio-temporal model. For instance, in the most recent earthquake model in California, geophysicists have already considered both the ETAS and the fault model/distribution \cite {Field2017}. 




\section*{Acknowledgement}

The work is partially funded by an NSF CAREER Award CCF-1650913, NSF CMMI-2015787, DMS-1938106, DMS-1830210.

\bibliography{refs}
\bibliographystyle{IEEEtran}
\appendix

\begin{proof}[Proof of Proposition \ref{prop1}]
\begin{equation*}
\begin{split}
& ~\int_{t_i}^{t_{i+1}} \int_{\mathcal{S}} \lambda^*_{\theta}(\tau, r) dr d\tau \\ 
= & ~\int_{t_i}^{t_{i+1}} \int_{\mathcal{S}} \left[ \lambda_0 + \sum_{j:t_j < \tau} \sum_{k=1}^K \phi^{(k)}_{s_j} \cdot g(\tau, t_j, r, s_j) \right ] dr d\tau \\
= &~ \lambda_0 (t_{i+1} - t_i) |\mathcal{S}| + (1-\epsilon) \\
&~\sum_{j:t_j < t_{i+1}} \sum_{k=1}^K \phi^{(k)}_{s_j} \int_{t_i}^{t_{i+1}} \int_{\mathbb R^2} g(\tau, t_j, r, s_j) dr d\tau
\end{split}
\end{equation*}
where $|\mathcal{S}|$ is the area of the space region. The last equality is breaking the integral into two parts, over set $\mathcal S$ and its complement. 

The triple integral in the second term can be written explicitly by changing variable. Let $q$ be the radius of the oval, $\varphi$ be the angle, so that we have $u/\sigma_x^{(k)}(s_i) = q \cos(\varphi)$, $v/\sigma_y^{(k)}(s_i) = q \sin(\varphi)$, where $(u,v) \in \mathbb{R}^2$ are the Cartesian coordinate. The Jacobian matrix of this variable transformation is \[
\begin{bmatrix}
\sigma_{x}^{(k)}(s_i) \cos (\varphi) & 
-\sigma_{x}^{(k)}(s_i) q\sin (\varphi) \\
\sigma_{y}^{(k)}(s_i) \sin (\varphi) & 
\sigma_{y}^{(k)}(s_i) q\cos (\varphi)
\end{bmatrix}
\]
So the determinant of the Jacobian is $q \sigma_{x}^{(k)}(s_i)\sigma_{y}^{(k)}(s_i)$.

Therefore, the double integral, which is independent from $q$ and $\varphi$, can be written as
\begin{align*}
& \int_{t_i}^{t_{i+1}} \int_{\mathbb{R}^2} g(\tau, t_j, r, s_j) dr d\tau \\
= &~ \int_{t_i}^{t_{i+1}} \int^{2\pi}_0 \int^{\infty}_0 \frac{C e^{-\beta (\tau - t_j)}}{2\pi |\Sigma_{s_j}|^{-1/2} (\tau - t_j)} \cdot\\
&~\exp\left\{-\frac{q^2}{2(\tau-t_j)} \right\} \cdot q \sigma_{x}^{(k)}(s_j)\sigma_{y}^{(k)}(s_j) dq d\varphi d\tau \\
= &~ \frac{C \sigma_{x}^{(k)}(s_j)\sigma_{y}^{(k)}(s_j)}{2\pi |\Sigma_{s_j}|^{-1/2}} \int_{t_i}^{t_{i+1}} \int^{2\pi}_0 \int^{\infty}_0 \frac{ e^{-\beta (\tau - t_j)}}{(\tau - t_j)}\\
&~\exp\left\{-\frac{q^2}{2(\tau-t_j)}\right\} q dq d\varphi d\tau\\
= &~ \frac{C \sigma_{x}^{(k)}(s_j)\sigma_{y}^{(k)}(s_j)}{2\pi |\Sigma_{s_j}|^{-1/2}} \int_{t_i}^{t_{i+1}} \int^{2\pi}_0 \exp\left\{ - \beta (\tau - t_j) \right\} d\varphi d\tau\\
= &~ \frac{C \sigma_{x}^{(k)}(s_j)\sigma_{y}^{(k)}(s_j)}{|\Sigma_{s_j}|^{-1/2}} \int_{t_i}^{t_{i+1}} \exp\left\{ - \beta (\tau - t_j) \right\} d\tau\\
= &~ \frac{C \sigma_{x}^{(k)}(s_j)\sigma_{y}^{(k)}(s_j)}{\beta |\Sigma_{s_j}|^{-1/2}} \left(e^{-\beta(t_{i} - t_{j})} - e^{-\beta(t_{i+1}-t_{j})} \right).
\end{align*}
Let  
\[
    C_j = \sum_{k=1}^K \phi_{s_j}^{(k)} \frac{\sigma_{x}^{(k)}(s_j)\sigma_{y}^{(k)}(s_j)}{|\Sigma^{(k)}_{s_j}|^{1/2}},
\]
we can have 
\begin{align*}
    &~\int_{t_i}^{t_{i+1}} \int_{\mathcal{S}} \lambda^*_{\theta}(\tau, r) dr d\tau = \lambda_0 (t_{i+1} - t_i) |\mathcal{S}| + \\
    &~(1-\epsilon) \frac{C}{\beta} \sum_{j:t_j < t_{i+1}} C_j \left(e^{-\beta(t_{i} - t_{j})} - e^{-\beta(t_{i+1}-t_{j})} \right),
\end{align*}
where the constant
\[
    \epsilon =\max_{j: t_j < t_{i+1}} \frac{\int_{t_i}^{t_{i+1}} \int_{\mathcal S} g(\tau, t_j, r, s_j) dr d \tau}{\int_{t_i}^{t_{i+1}}\int_{\mathbb R^2} g(\tau, t_j, r, s_j) dr d\tau}.
\]

\end{proof}

\begin{proof}[Proof of Proposition \ref{prop_policy}]
Now we can derive a consequence of the above lemma using the following simple argument. Let 
\[F(t) = \int_{0}^t \int_{\mathcal S} \pi (\tau, r) d\tau dr .\] Then 
\[\dot F(t) = \frac{d F(t)}{dt} = \int_{\mathcal S} \pi (t, r) dr.\] From (\ref{lambda_star}) we obtain 
\begin{equation*}
\begin{split}
 \int_{\mathcal S} \lambda_\theta^* (t, r) dr
= &~\frac{\int_{\mathcal S} \pi (t, r) dr}{1 - \int_{0}^{t}\int_{\mathcal{S}} \pi(\tau, r) d\tau dr}\\
= &~\frac{\dot F(t)}{1-F(t)} = - \frac{d }{dt}(\log (1-F(t)))
\end{split}
\end{equation*}
Thus, 
\[\int_{t_n}^t \int_{\mathcal S}\lambda_{\theta}^* (\tau, r) d\tau dr = 
\log(1-F(t_n)) - \log(1-F(t)).
\]
Since $F(t_n) = 0$, we obtain 
\[1-F(t) = \exp\{-\int_{t_n}^t \int_{\mathcal S}\lambda_{\theta}^* (\tau, r) d\tau dr\}.\]
Also, from (\ref{lambda_star}), we can write 
\[
\lambda^*_\theta(t, s) = \frac{\pi(t, s)}{1- F(t)}.
\]
Thus 
\begin{equation*}
\begin{split}
\pi(t, s) = &~\lambda^*_\theta(t, s)(1- F(t)) \\
= &~\lambda^*_\theta(t, s) \exp\{-\int_{t_n}^t \int_{\mathcal S}\lambda_{\theta}^* (\tau, r) d\tau dr\}.
\end{split}
\end{equation*}
\end{proof}

\begin{figure}[!t]
\centering
\begin{subfigure}{1.\linewidth}
\includegraphics[width=\linewidth]{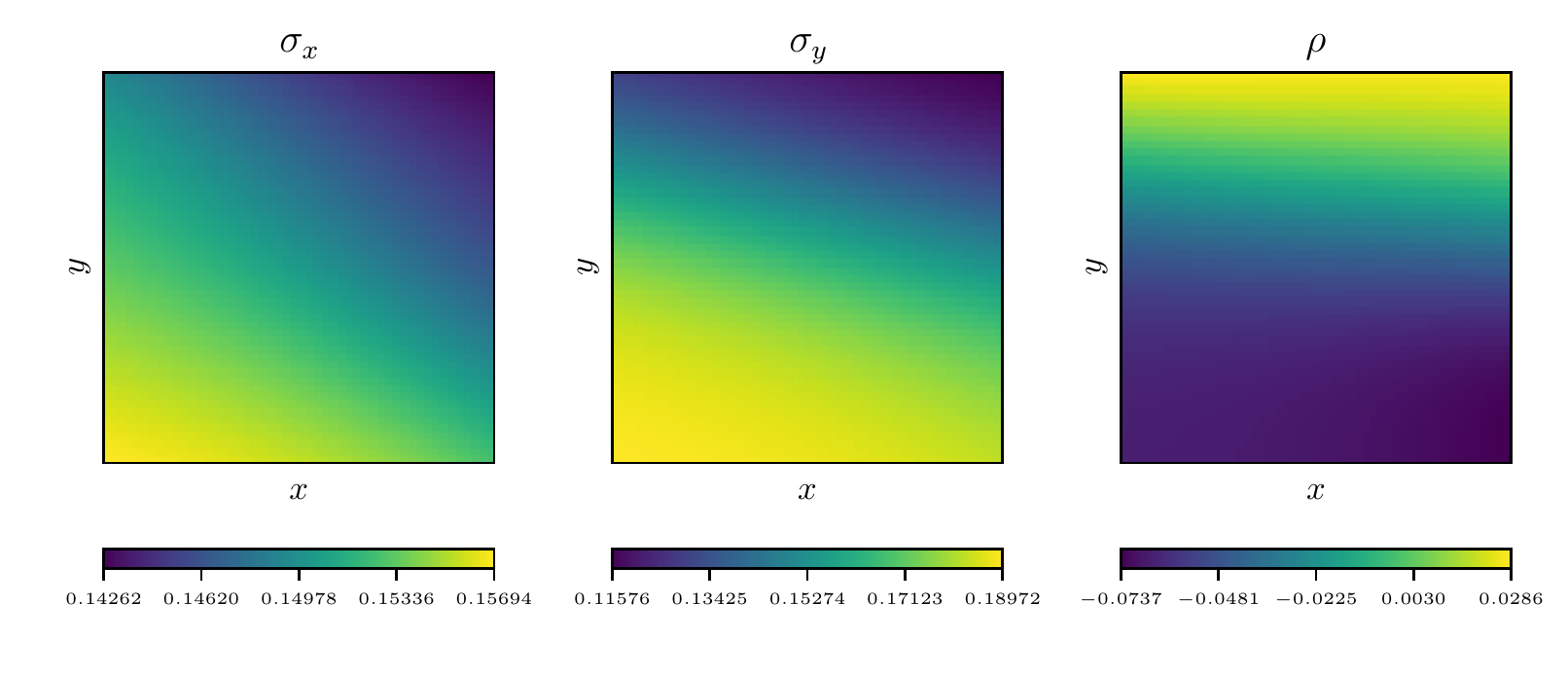}
\caption{1st component}
\end{subfigure}
\begin{subfigure}{1\linewidth}
\includegraphics[width=\linewidth]{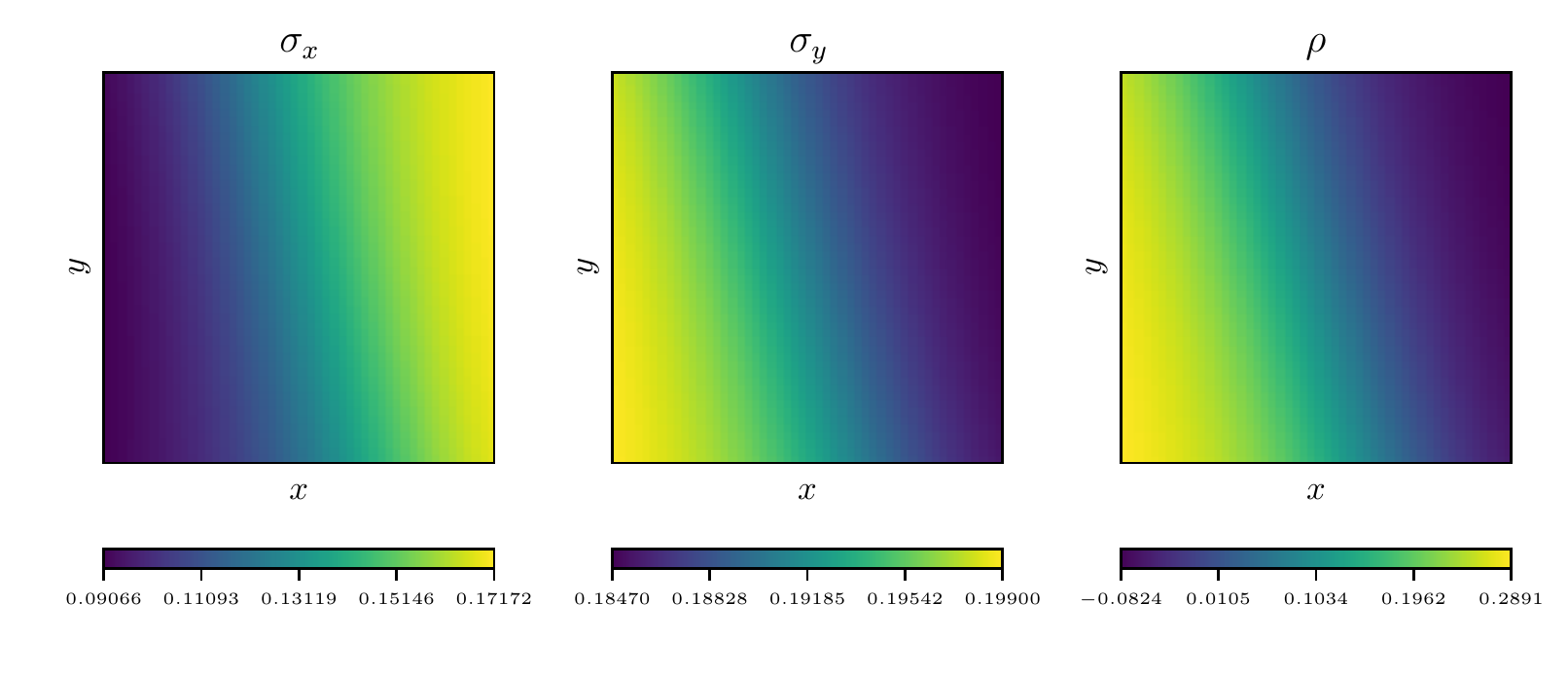}
\caption{2nd component}
\end{subfigure}
\begin{subfigure}{1\linewidth}
\includegraphics[width=\linewidth]{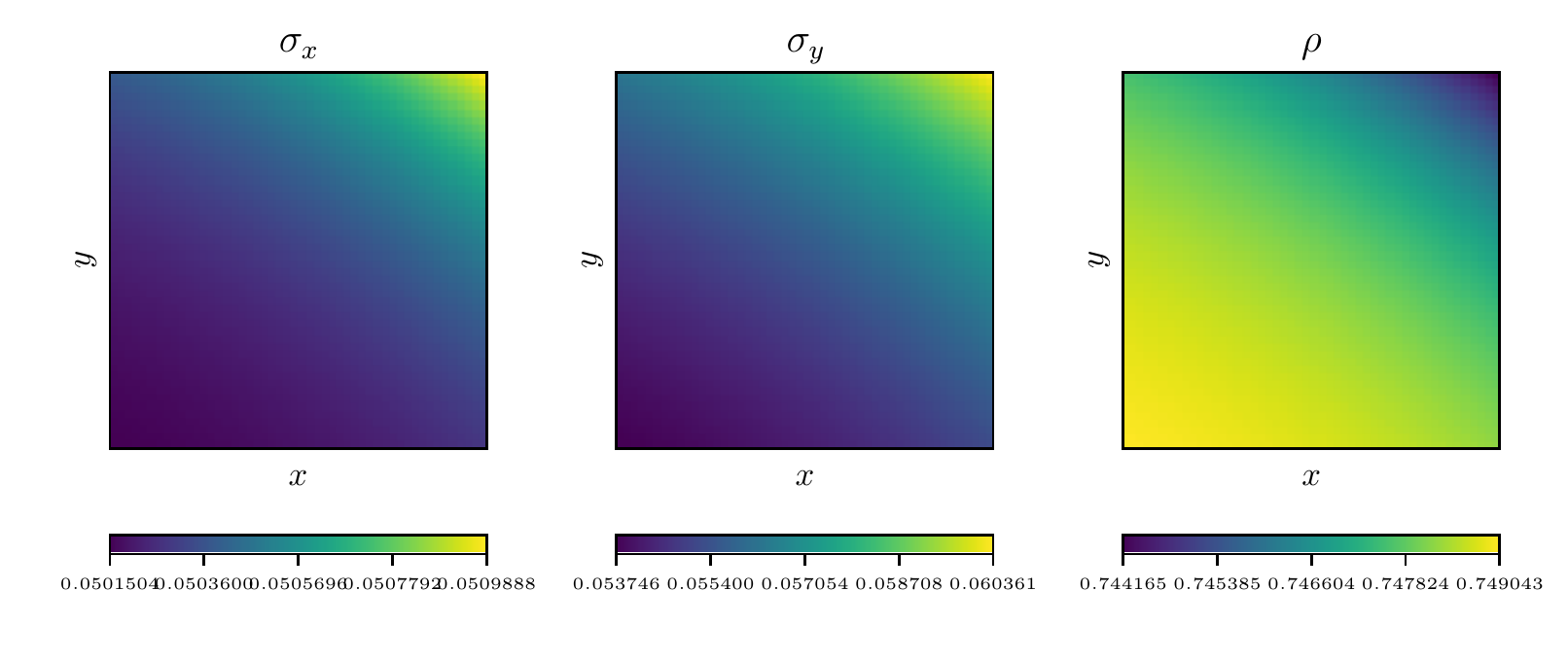}
\caption{3rd component}
\end{subfigure}
\begin{subfigure}{1\linewidth}
\includegraphics[width=\linewidth]{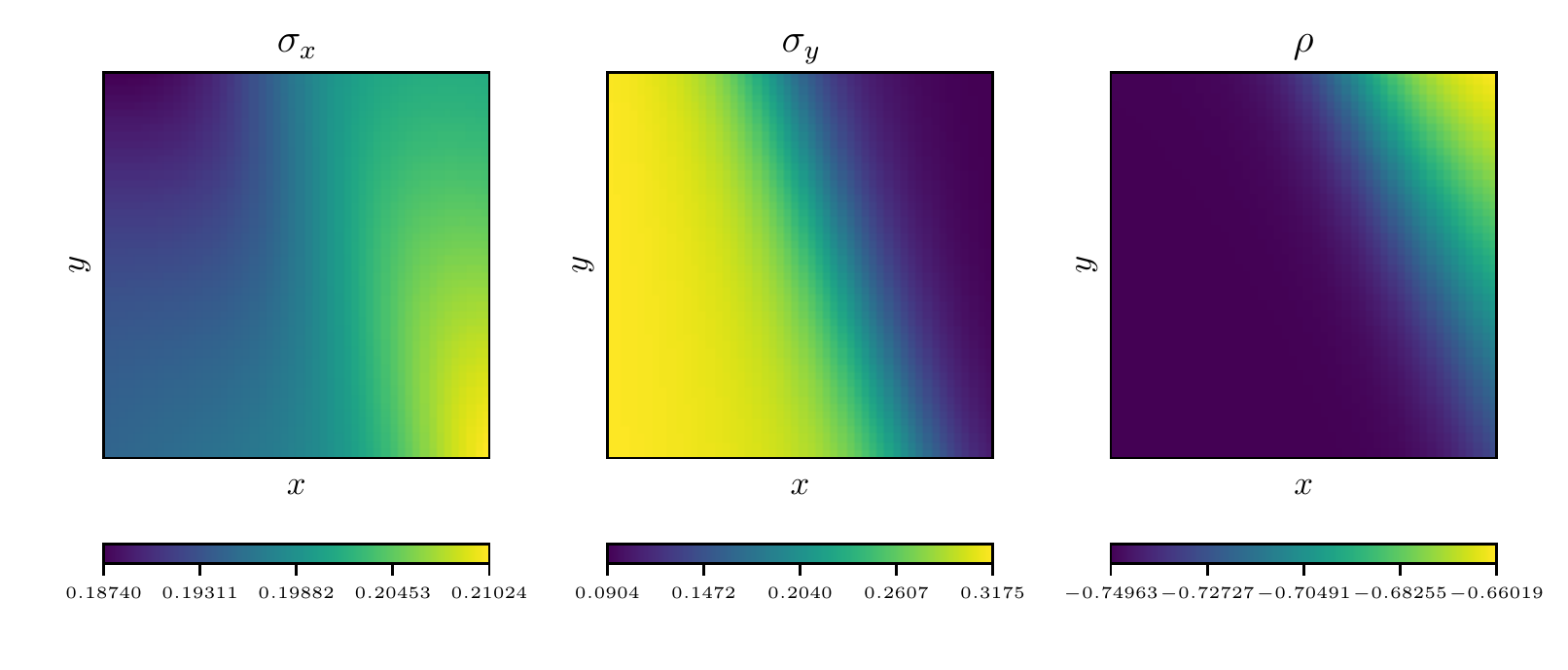}
\caption{4th component}
\end{subfigure}
\begin{subfigure}{1\linewidth}
\includegraphics[width=\linewidth]{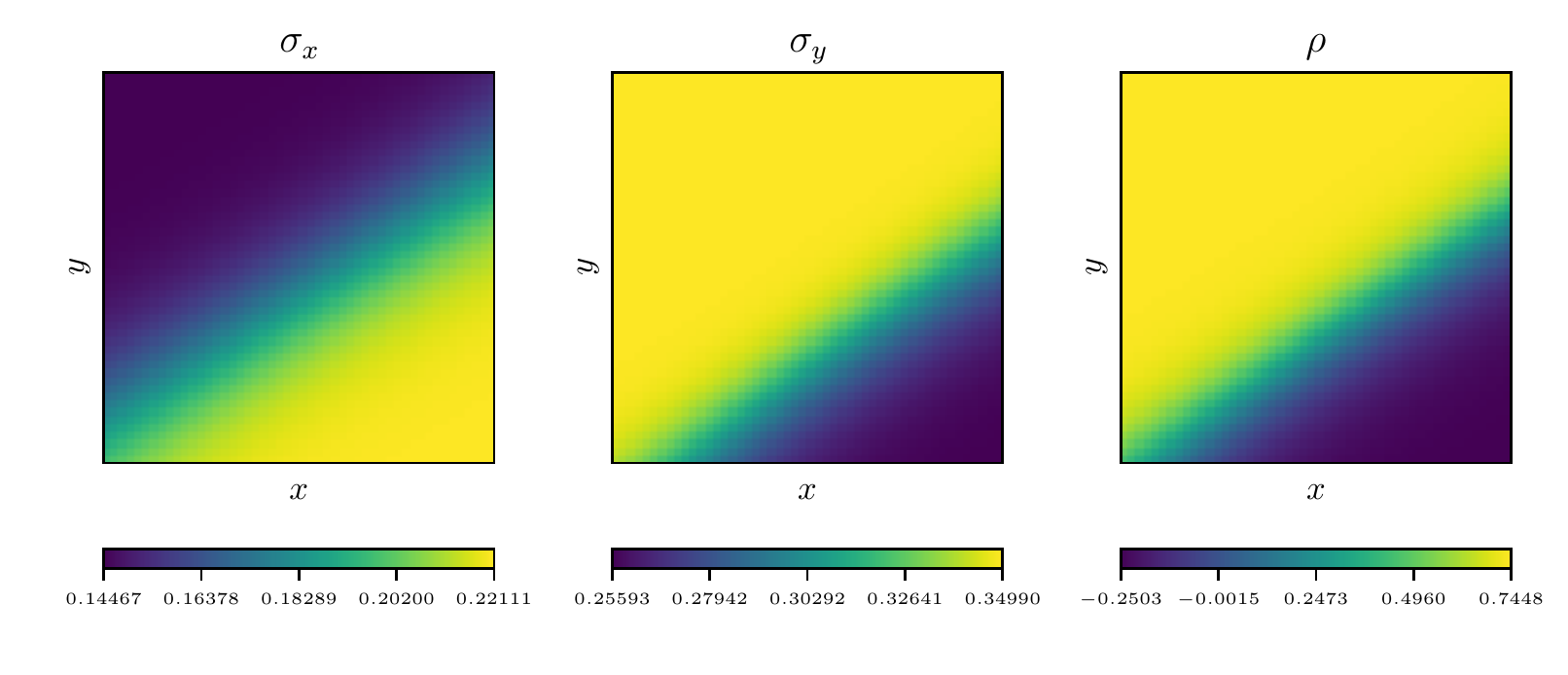}
\caption{5th component}
\end{subfigure}
\caption{Parameters of each Gaussian components in the NEST model ($K=5$) fitted using the Northern California earthquake data set.}
\label{fig:param-earthquake}
\end{figure}

\begin{figure}[!t]
\centering
\begin{subfigure}{1.\linewidth}
\includegraphics[width=\linewidth]{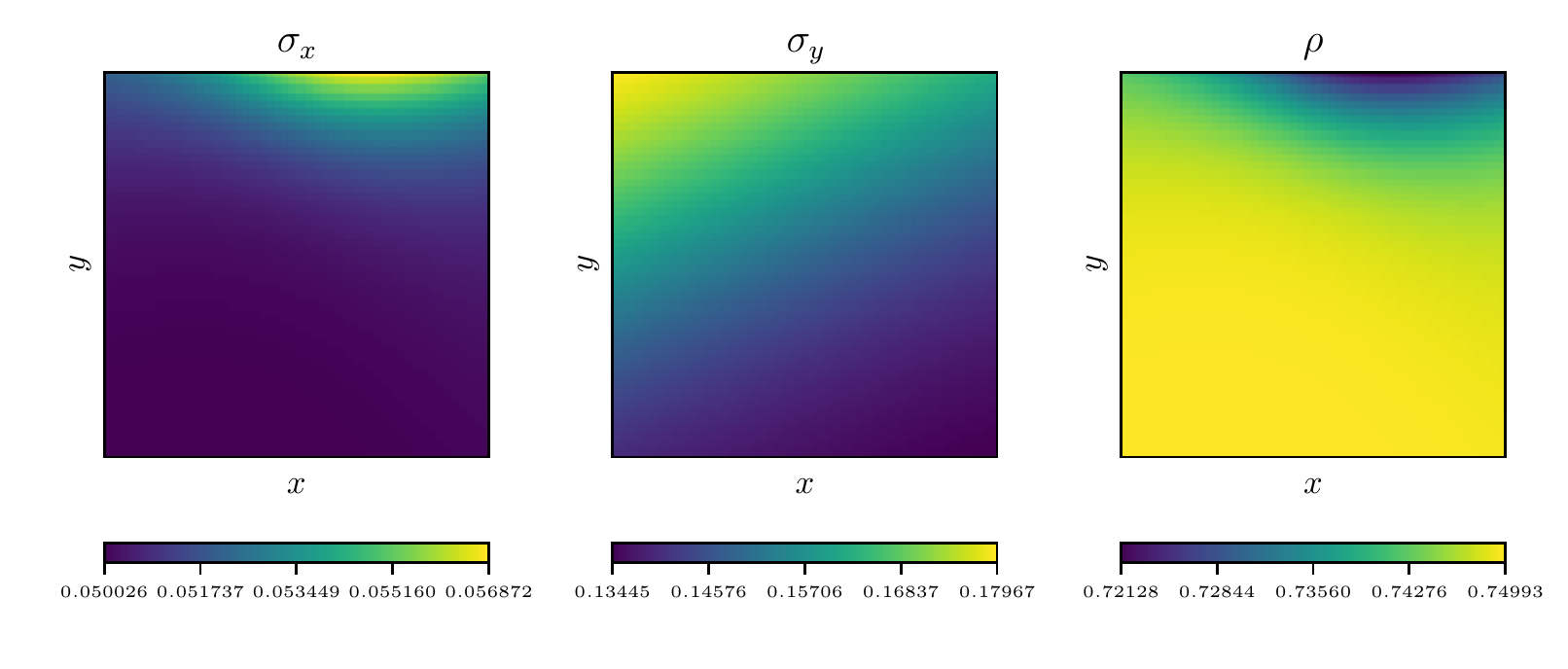}
\caption{1st component}
\end{subfigure}
\begin{subfigure}{1\linewidth}
\includegraphics[width=\linewidth]{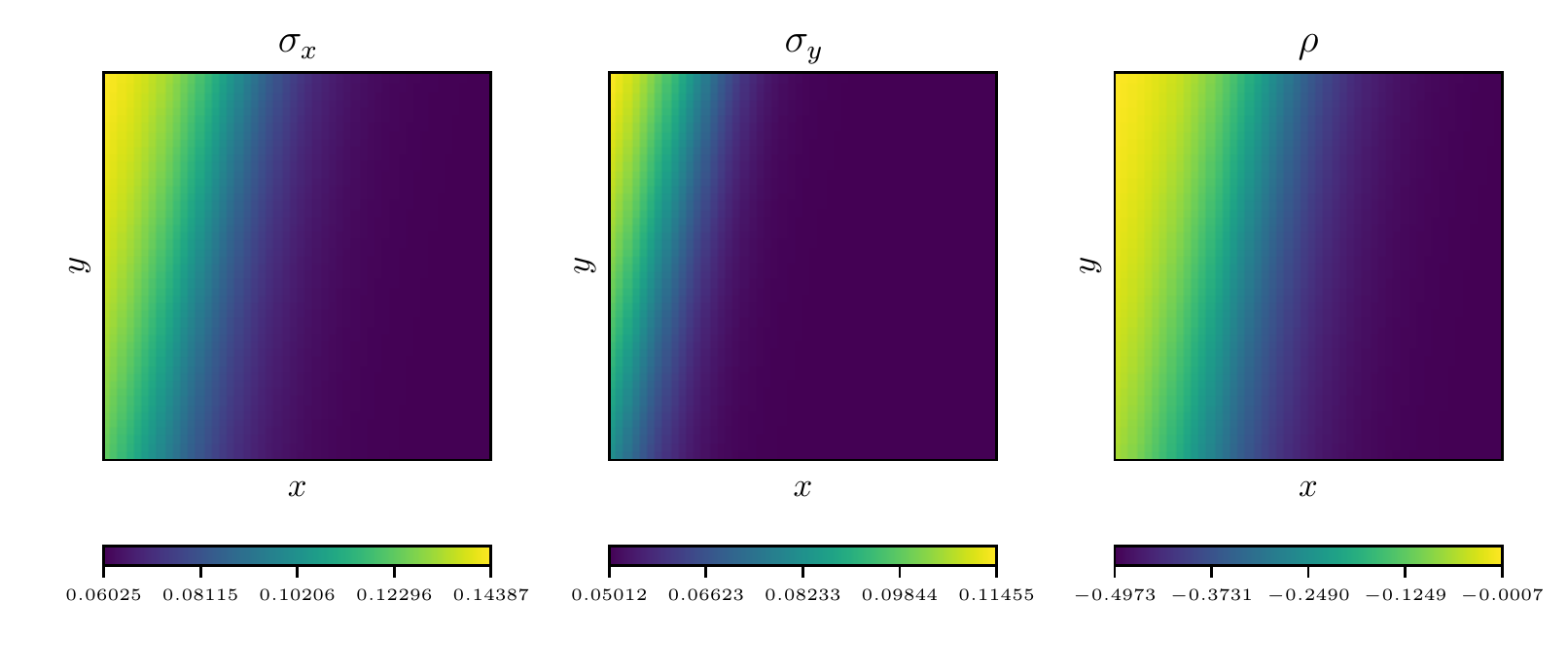}
\caption{2nd component}
\end{subfigure}
\begin{subfigure}{1\linewidth}
\includegraphics[width=\linewidth]{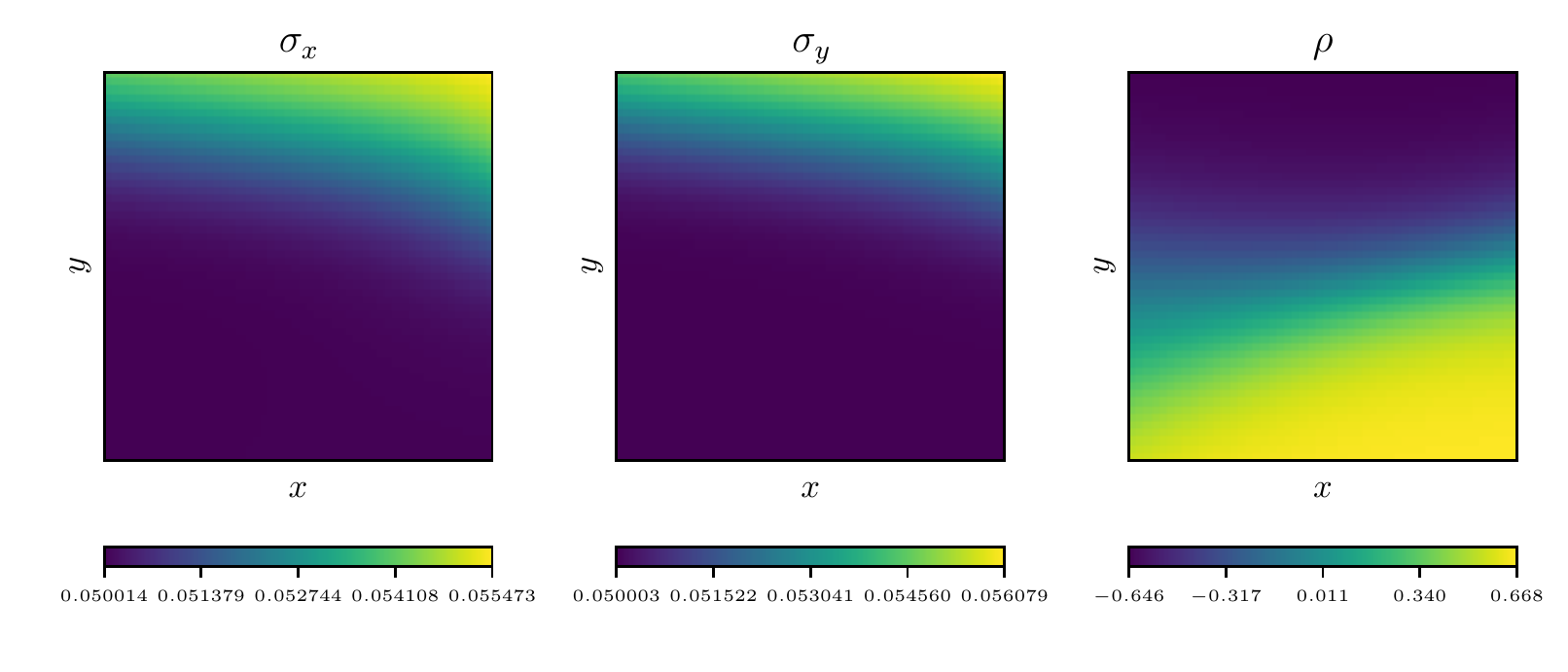}
\caption{3rd component}
\end{subfigure}
\begin{subfigure}{1\linewidth}
\includegraphics[width=\linewidth]{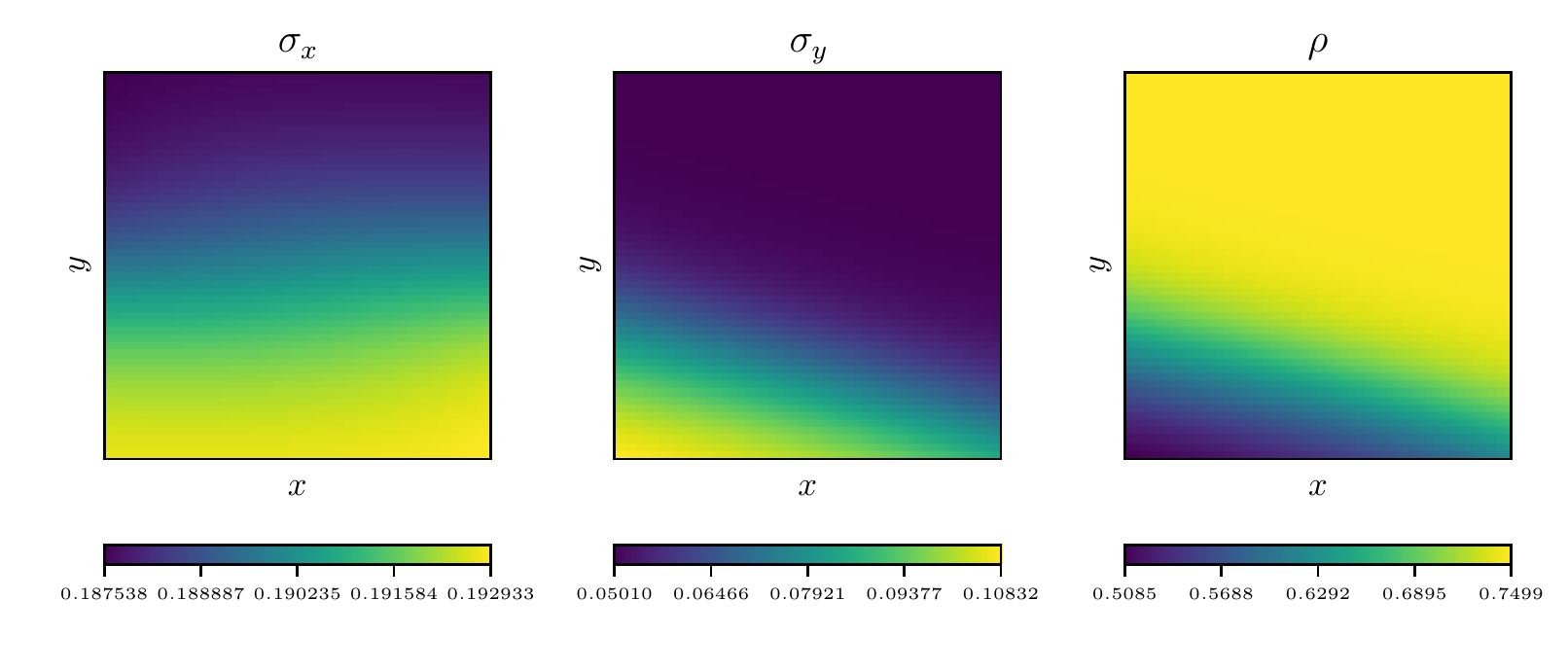}
\caption{4th component}
\end{subfigure}
\begin{subfigure}{1\linewidth}
\includegraphics[width=\linewidth]{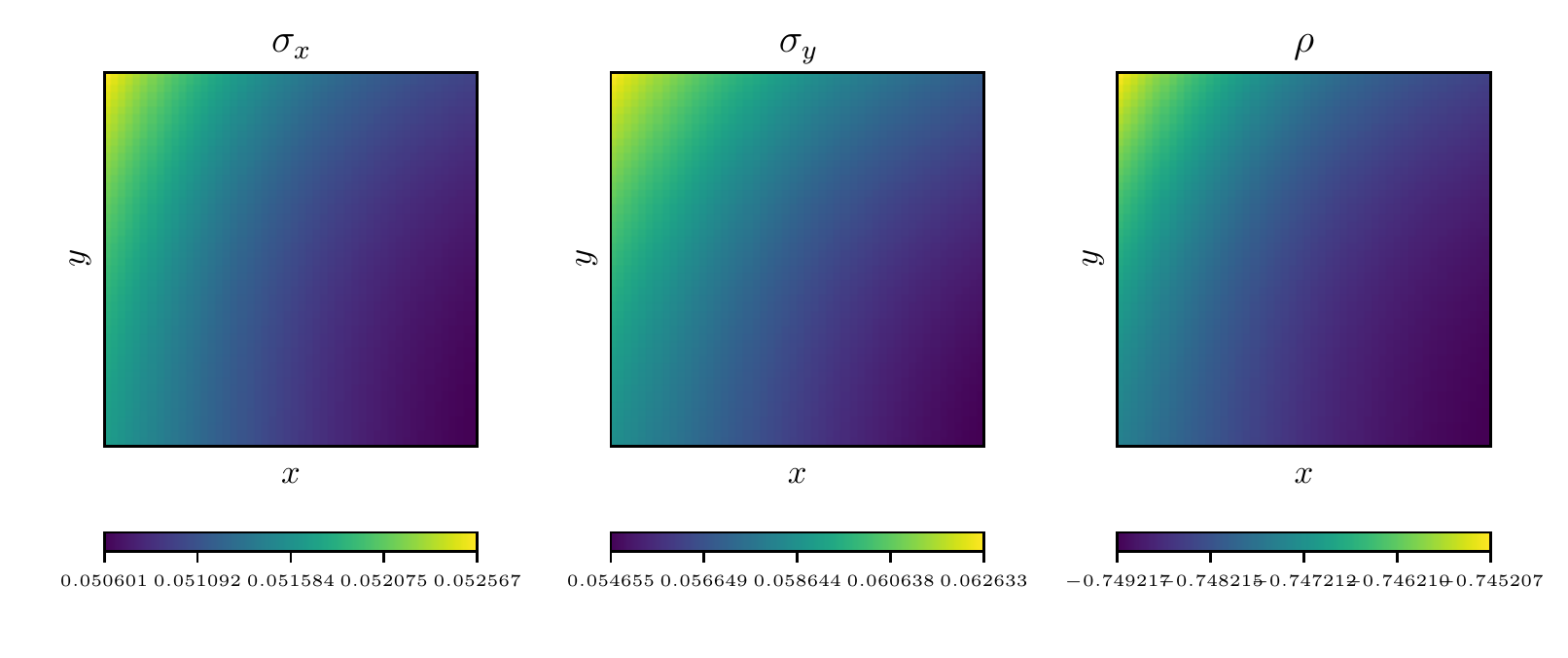}
\caption{5th component}
\end{subfigure}
\caption{Parameters of each Gaussian component in the NEST model ($K=5$) fitted using the 911 calls-for-service robbery data set.}
\label{fig:param-robbery}
\end{figure}

\end{document}